\documentclass[conference]{IEEEtran}
\usepackage{times}
\usepackage{balance}

\usepackage{balance}
\usepackage{comment}

\usepackage[utf8]{inputenc} 
\usepackage[T1]{fontenc}    
\usepackage{url}            

\usepackage{booktabs}       
\usepackage[labelfont=bf,font=small]{caption}

\usepackage{amsfonts}       
\usepackage{nicefrac}       
\usepackage{microtype}      
\usepackage{graphicx}
\usepackage[noadjust]{cite}

\usepackage{amsthm}

\usepackage{amsmath}
\usepackage{algorithm}
\usepackage{algpseudocode}
\usepackage[dvipsnames]{xcolor}

\usepackage{subfigure}
\usepackage[Export]{adjustbox}
\usepackage{tabularx}
\usepackage{cellspace}

\makeatletter
\let\NAT@parse\undefined
\makeatother
\usepackage{hyperref}       
\hypersetup{
    colorlinks,
    linkcolor={blue!90!black},
    citecolor={blue!90!black},
    urlcolor={blue!95!black}
}

\usepackage[capitalise]{cleveref}       
\usepackage{crossreftools}  
\usepackage[all=normal,paragraphs=tight,floats=tight,mathspacing=tight,wordspacing=tight,tracking=tight]{savetrees}

\usepackage{amsmath}
\usepackage{bm}

\usepackage{amsthm}
\newtheorem{theorem}{Theorem}
\newtheorem{lemma}{Lemma}

\newtheorem{corollary}{Corollary}
\newtheorem{proposition}{Proposition}    
    
\newtheorem{definition}{Definition}

\usepackage{ifthen}
\newboolean{include-notes}
\newboolean{include-new}
\newboolean{include-remove}
\setboolean{include-notes}{true}
\setboolean{include-new}{true}
\setboolean{include-remove}{false}

\usepackage{xcolor}
\usepackage[normalem]{ulem}
\definecolor{newcolor}{rgb}{0, 0, 1}
\definecolor{removecolor}{rgb}{1, 0, 0}
\newcommand{\jaime}[1]{\ifthenelse{\boolean{include-notes}}{\textcolor{orange}{\textbf{JFF:} #1}}{}}
\newcommand{\haimin}[1]{\ifthenelse{\boolean{include-notes}}{\textcolor{teal}{\textbf{HH:} #1}}{}}
\newcommand{\justin}[1]{\ifthenelse{\boolean{include-notes}}{\textcolor{Peach}{\textbf{Justin:} #1}}{}}
\newcommand{\maria}[1]{\ifthenelse{\boolean{include-notes}}{\textcolor{violet}{\textbf{Maria:} #1}}{}}
\newcommand{\remove}[1]{\ifthenelse{\boolean{include-remove}}{\textcolor{removecolor}{\sout{#1}}}{}}
\newcommand{\new}[1]{\ifthenelse{\boolean{include-new}}{\textcolor{newcolor}{#1}}{#1}}

\definecolor{porange}{HTML}{E77500} 

\usepackage{lipsum}

\newcommand{\mean}{\mu}

\newcommand{\covar}{\Sigma}

\newcommand{\gaussian}{{\mathcal{N}}}

\newcommand{\expectation}{{\mathbb{E}}}
\newcommand{\mixturecoeff}{w}
\newcommand{\mode}{m}

\DeclareMathOperator*{\argmax}{arg\,max}

\newcommand{\trace}{\operatorname{tr}}

\newcommand{\logdet}{\operatorname{logdet}}

\newcommand{\state}{{x}}

\newcommand{\ctrl}{{u}}
\newcommand{\dstb}{{d}}

\newcommand{\nominaltraj}{{\eta}}

\newcommand{\dyn}{{f}}

\newcommand{\lterm }{z^i_{t+1}}
\newcommand{\qterm }{Z^i_{t+1}}

\newcommand{\jxt}{x_t} 
\newcommand{\jut}{u_t} 

\newcommand{\valfunc}{{V}}
\newcommand{\qfunc}{{\mathcal{Q}}}

\newcommand{\policy}{{\pi}}
\newcommand{\policyref}{{\tilde{\pi}}}
\newcommand{\policyrefGain}{{\tilde{K}}}
\newcommand{\policyrefOffset}{{\tilde{\kappa}}}

\newcommand{\reg}{{\lambda}}

\newcommand{\nodeset}{\mathcal{NS}}

\newcommand{\node}{n}
\newcommand{\tnode}{{\tilde{\node}}}

\newcommand{\rebuttal}[1]{#1}

\title{Blending Data-Driven Priors in Dynamic Games}
\IEEEoverridecommandlockouts

\author{Justin Lidard$^{1,*}$, \thanks{$^{1}$Department of Mechanical and Aerospace Engineering, Princeton University, Princeton, NJ 08540, USA} Haimin Hu$^{2,*}$ \thanks{$^{2}$Department of Electrical and Computer Engineering, Princeton University, Princeton, NJ 08540, USA}, Asher Hancock$^{1,\dagger}$, Zixu Zhang$^{2,\dagger}$, \\ \textnormal{Albert Gim\'o Contreras$^{2}$, Vikash Modi$^{2}$, Jonathan A. DeCastro$^{3}$, Deepak Gopinath$^{3}$,} \\ \textnormal{ Guy Rosman$^{3}$, \thanks{$^{3}$Toyota Research Institute, Cambridge, MA 02139, USA} Naomi Ehrich Leonard$^{1}$, Mar\'ia Santos$^{1}$, and Jaime Fern\'andez Fisac$^{2}$  \thanks{This research has been supported in part by an NSF Graduate Research Fellowship. This work is partially supported by Toyota Research Institute (TRI). It, however, reflects solely the opinions and conclusions of its authors and not TRI or any other Toyota entity.}} \thanks{$^{*}$J. Lidard and H. Hu contributed equally.}\thanks{ $^{\dagger}$A. Hancock and Z. Zhang contributed equally. }}
\begin{document}

\maketitle

\begin{abstract}
    As intelligent robots like autonomous vehicles become increasingly deployed in the presence of people,
    the extent to which these systems should leverage 
    model-based game-theoretic planners versus data-driven policies for safe, interaction-aware motion planning remains an open question.
    Existing dynamic game formulations assume all agents are task-driven and behave optimally.
    However, in reality, humans tend to deviate from the decisions prescribed by these models, and their behavior is better approximated under a \textit{noisy-rational} paradigm. 
    %
    In this work, we investigate a principled methodology to blend a data-driven \textit{reference policy} with an optimization-based game-theoretic policy.
    We formulate \textit{KLGame}, an algorithm for solving non-cooperative dynamic game with Kullback-Leibler (KL) regularization with respect to a general, stochastic, and possibly multi-modal reference policy.
    Our method incorporates, for each decision maker, a tunable parameter that permits \textit{modulation} between task-driven and data-driven behaviors.
    We propose an efficient algorithm for computing multi-modal approximate feedback Nash equilibrium strategies of KLGame in real time.
    Through a series of simulated and real-world autonomous driving scenarios, we demonstrate that 
    KLGame policies can more effectively incorporate guidance from the reference policy and account for noisily-rational human behaviors versus non-regularized baselines. Website with additional information, videos, and code: \href{https://kl-games.github.io/}{https://kl-games.github.io/}.
\end{abstract}

\section{Introduction}


 \begin{figure}
    \centering
    \includegraphics[width=0.5\textwidth]{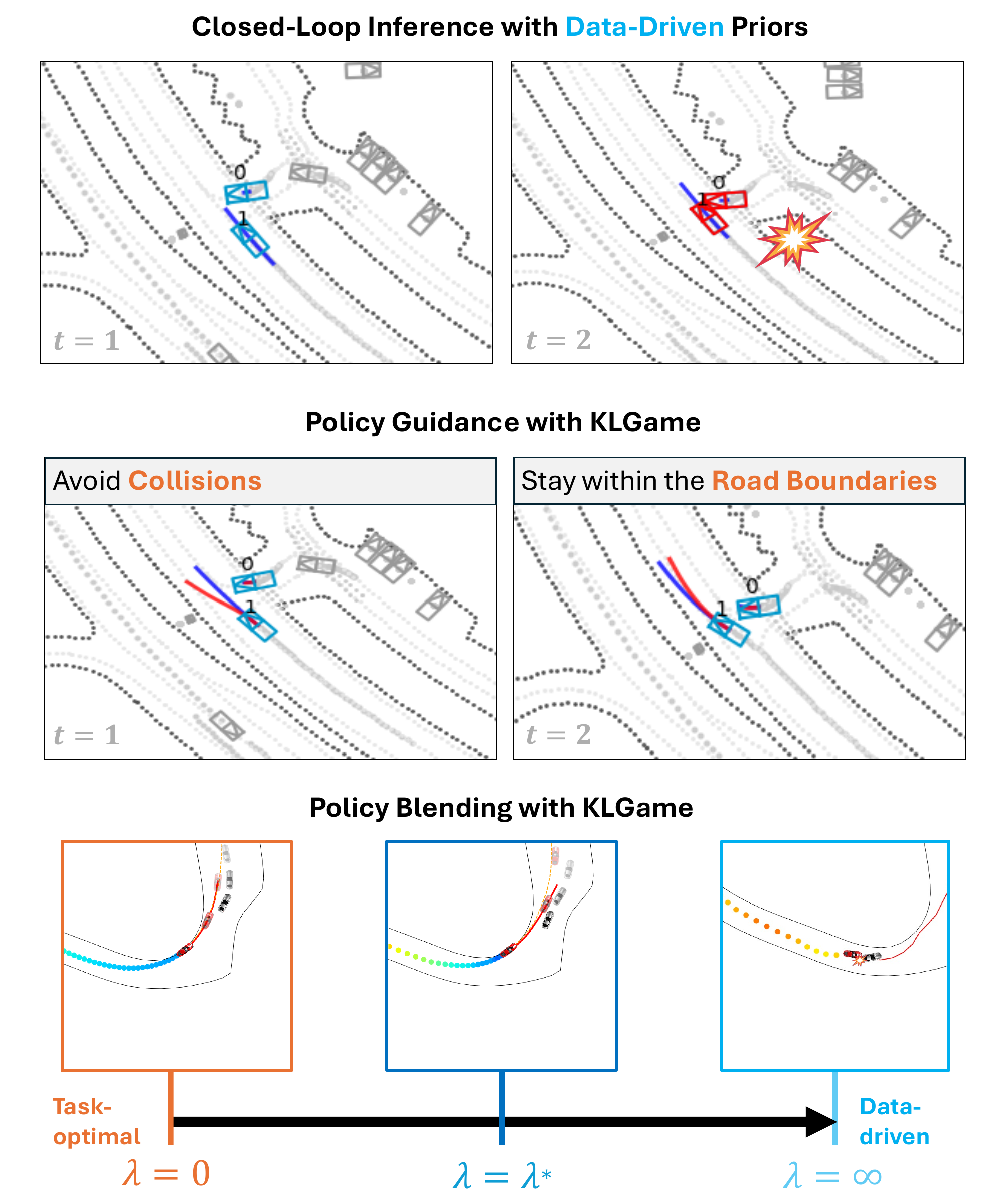}
    \caption{\rebuttal{Our proposed approach can seamlessly integrate data-driven policies with optimization-based dynamic game solutions. \textit{Top:} Data-driven behavior predictors, such as transformer-based methods, provide \textit{marginal} (i.e., single-agent) priors for human motion prediction, but may struggle to model \textit{closed-loop}, multi-agent interactions. 
    \textit{Middle:} Our proposed method, KLGame, allows a robot to incorporate \textit{guidance} from data-driven rollouts while performing online game-theoretic planning in closed-loop.
    \textit{Bottom:} KLGame uses a \textit{tunable} parameter $\lambda$ that modulates behaviors on a spectrum: $\lambda=0$ gives a deterministic dynamic game (\textit{task-optimal}), and $\lambda\rightarrow \infty$ gives multi-modal behavior cloning. We call this tunability \textit{policy blending}. }}
    \vspace{-0.6cm}
    \label{fig: Rollouts}
\end{figure}

Planning safe trajectories in the presence of multiple decision-making agents is a long-standing challenge in robotics. Desired performance and safety behaviors can often be encapsulated in mathematical models and algorithms that govern the operation of the agents, offering analytically tractable guarantees. However, first-principles planners may fall short in capturing the nuanced variety of scenarios that are subject to arise in interactive settings. In turn, data-driven priors can be extremely informative in these situations, with \textit{policy alignment} strategies allowing agents to incorporate fine-tuned, scenario-specific heuristics as well as human-desired intent into their motion planning. 

Two diverging approaches have recently emerged for planning in the presence of nearby agents. Dynamic game theory~\cite{bacsar1998dynamic}, which models interactions as a multi-agent optimal control problem, allows computing a task-optimal plan simultaneously for all agents in the scene. However, while these methods can effectively capture high-level rational behaviors, other agents do not always act rationally---i.e., they do not necessarily pick the game-theoretically optimal action. Moreover, gradient-based game-theoretic solvers may struggle to model \textit{multi-modal}, \textit{mixed-strategy} (stochastic) interactions, wherein distinct high-level behaviors may arise from the same objectives, leading to several local optima in the space of rewards. Conversely, behavior cloning methods provide a framework for end-to-end learning of multi-modal trajectory distributions from data, but may yield suboptimal or unsafe predictions. Such data-driven methods exploit recent advances in autonomous vehicle perception systems by synthesizing high-dimensional sensor inputs, including a road map, traffic light states, and LIDAR data. While imitation learning can provide a flexible and efficient approximation for interactions modeled by dynamic game theory, predictions may be suboptimal or unsafe.

Integrated prediction and planning approaches \cite{espinoza2022deep, lindemann2023safe} are a popular middle-ground approach that combines first principles from game-theoretic planning with data-driven models derived from human motion. In this paradigm, a deep-learned motion forecasting system first produces a trajectory prediction before planning begins, reducing the computational load on the planner. However, substantially different behaviors may arise depending on whether agents adhere more strongly to the data-driven or optimal behavior. Fig.~\ref{fig: Rollouts} provides an example of a multi-modal interaction where a human neglects an oncoming vehicle due to a visual occlusion. Since the vehicle's goal is near the human, the human's safety is directly in conflict with the target of the vehicle. Furthermore, the most likely motion forecast for both agents predicts that both will go straight, necessitating intervention to avoid a collision. Neither the game-theoretic nor imitation-learned plan provides a safe trajectory for both agents,
calling for another method.

In this work, we introduce KLGame, a novel multi-modal game-theoretic planning framework that incorporates data-driven policy models (such as motion forecasting) into each agent's policy optimization through a regularization bonus in the planning objective. KLGame incentivizes the trajectories of planning agents to not only optimize hand-crafted cost heuristics, but also to adhere to a \textit{reference policy}. In this work, we assume that the reference policy is distilled from data or expert knowledge, is stochastic, and may be multi-modal in general. In contrast to other integrated prediction and planning methods, KLGame (i) provides an analytically and computationally sound methodology for planning under strategy uncertainty, exactly solving the regularized stochastic optimization problem; and (ii) incorporates \textit{tunable} multi-modal, data-driven motion predictions in the optimal policy through a scalar parameter, allowing the planner to modulate between purely data-driven and purely optimal behaviors.

\noindent{\textbf{Statement of contributions.}} We make two key contributions:
\begin{itemize}
    \item We introduce \textit{KLGame}, a novel stochastic dynamic game that blends \textit{interaction-aware} task optimization with \textit{closed-loop} policy guidance via Kullback-Leibler (KL) regularization. We provide an in-depth analysis in the linear-quadratic (LQ) setting with Gaussian reference policies and show KLGame permits an analytical \textit{global feedback Nash equilibrium}, which naturally generalizes the solution of the maximum-entropy game~\cite{mehr2023maximum}. 
    \item We propose an efficient and scalable trajectory optimization algorithm for computing approximate feedback Nash equilibria of KLGame with general nonlinear dynamics, costs, and \textit{multi-modal} reference policies. Experimental results on Waymo's Open Motion Dataset demonstrate the efficacy of KLGame in leveraging data-driven priors compared to state-of-the-art methods~\cite{mehr2023maximum,fridovich2020efficient,shi2023mtr++}.
\end{itemize}

\section{Related Work}
\label{sec:related_work}
Our work relates to recent advances in game-theoretic motion planning, stochastic optimal control, and multi-agent trajectory prediction.

\subsection{Game-Theoretic Motion Planning}
\label{subsec:game_theoretic_motion_planning}

Game-theoretic motion planning is a popular choice for modeling multi-agent non-cooperative interactions, such as autonomous driving \cite{fridovich2020efficient, mehr2023maximum, schwarting2019social, wang2021game, hu2023activeIJRR}, crowd navigation \cite{sun2021move}, distributed systems \cite{williams2023distributed, hu2020non}, drone racing \cite{spica2020real}, and shared control \cite{music2020haptic}. A central focus of game-theoretic planning is local Nash equilibria (LNE) \cite{bacsar1998dynamic}, in which no agent is unilaterally incentivized to deviate from their strategy. LNEs have been studied extensively in autonomous driving as a model for interaction in merging \cite{geiger2021learning,cleac2020algames,liu2023learning}, highway overtaking \cite{le2021lucidgames,Fisac2019-hierarchical,hu2022activeWAFR}, and racing \cite{schwarting2021stochastic,wang2021game}. However, in many real-world settings, agents are not perfectly rational and may deviate significantly from actions they believe are optimal due to distractions \cite{steinberger2017road}, imperfect information \cite{talebpour2015modeling, hu2023belgame}, or poor human-machine interfacing \cite{young2017toward}. In this work, we adopt the principle of \textit{bounded rationality} \cite{ziebart2008maximum, wulfmeier2017large, evens2021, phan2022driving}, in which agents act rationally with likelihood proportional to the distribution of utilities over actions. 

While LNEs provide a mechanism for predicting and analyzing multi-agent interactions, multiple LNEs may arise depending on the initial condition of the joint (multi-agent) system \cite{bacsar1998dynamic}. Scenarios where multiple equilibria exist are called \textit{multi-modal} due to the existence of distinct (and possibly stochastic) outcomes, the selection of which might depend on individual preferences. Multi-modality in planning has recently become a popular object of study for modeling high-level decision-making in games. Peters et al.~\cite{peters2020inference} introduce an inference framework for maximum a posteriori equilibria aligned control by randomly selecting seed strategies and solving multiple seeds in parallel to an equilibrium, to then normalize their weights in a manner similar to a particle filter.  So et al. \cite{so2022maximum, so2023mpogames} show that equilibria can be efficiently clustered and inferred using a Bayesian update, which can then be used to hedge different strategies using a QMDP-style \cite{littman1995learning} cost estimate.
Peters et al.~\cite{peters2023contingency} show that multiple game-theoretic contingencies can be solved, with an arbitrary branching time, using mixed-complementarity programming. 
Recent work~\cite{hu2023activeIJRR} uses implicit dual control to tractably compute an active information gathering policy for a class of partially-observable stochastic games, while preserving the multi-modal information encoded in the robot's belief states with scenario optimization.
Follow-up work~\cite{hu2023belgame} extends this idea to a high-dimensional adversarial setting using deep reinforcement learning to synthesize the robot's policy at scale.

Multi-modality in planning can arise from uncertain or unknown objectives. For example, in a merging setting, which player prefers to go first may directly alter the outcome of the interaction, and this \textit{objective uncertainty} may lead to indecision.
Cleac’h et al. \cite{le2021lucidgames} introduce a motion planning method for uncertain linear cost models by estimating cost parameters using an unscented Kalman filter. Chen et al. \cite{chen2021interactive} introduce an MPC-based framework for interactive trajectory optimization using deep-learning predictions as input.
Hu et al. \cite{hu2023emergent} study emergent coordination in multi-agent interactions when cost parameters follow an opinion-dynamic interaction over a graph.
Recent work~\cite{peters2022rss} uses differentiable optimization to learn open-loop mixed strategies for non-cooperative games based on trajectory data.
Liu et al. \cite{liu2023learning} extend this idea to compute generalized equilibria, a setting where hard constraints are present, and use deep learning to bolster the computation performance of the game solver. Another follow-up work~\cite{li2023cost} further studies the setting of learning a feedback Nash equilibrium strategy. 
In the potential game setting, Diehl et al. \cite{diehl2023energy} show that multiple open-loop Nash equilibria can be computed in parallel using deep-learned game parameters. 

To the best of our knowledge, our approach is the first closed-loop, model-based game-theoretic planner that can incorporate multi-modal priors from data-driven trajectory prediction models for guided exploration and producing scene-consistent interactions. Our method makes only mild assumptions on the environment dynamics and planning costs (e.g. being differentiable). Moreover, our provision of a scalar regularization parameter permits \textit{tunable} adherence to data-driven behavior at inference time, in contrast to traditional game-theoretic planning and end-to-end equilibrium learning, which are rigid in their respective assumptions about planning objectives (i.e. task-driven or data-driven). 

\subsection{Stochastic Optimal Control and Dynamic Games}
\label{subseq:stochastic_optimal_control}


In contrast to greedy-optimal, deterministic behavior predicted by some classical control-theoretic planners, stochastic behavior has been exhibited in many real-world settings, such as biological muscle control \cite{todorov2004optimality}, animal foraging \cite{bartumeus2008fractal}, and urban driving \cite{ettinger2021waymo}. Stochastic optimal control (SOC) provides a mechanism for solving and explaining \textit{stochastic} optimal policies through accommodating uncertainty in decision-making. Energy-based methods \cite{kim2020hamilton, theodorou2012relative, haarnoja2017reinforcement, garg2021iq} are a well-studied technique for solving optimal policies by analyzing utility functions similar to the total energy in statistical physics. KL control \cite{todorov2006linearly, todorov2009compositionality, guan2014online, ito2022kullback} is a branch of energy-based methods that studies solutions for SOC problems where the control cost is augmented with a Kullback-Leibler divergence regularizer term that tries to keep the resultant policy close to some \textit{reference policy}, which may have different uses depending on the task. For example, the reference policy could minimize control effort \cite{todorov2006linearly}, guide exploration \cite{ok2018exploration}, or incorporate human feedback \cite{munos2023nash}.

Due to the intractability of the expectation term for general stochastic optimal control and dynamic games, scenario optimization~\cite{bernardini2011stabilizing,campi2018general} is often used as an approximate solution method.
Schildbach et al.~\cite{schildbach2015scenario} uses scenario optimization for autonomous lane change assistance when the ego vehicle is interacting with other human-controlled vehicles, whose future motions are predicted with a pre-specified scenario generation model and incorporated in a model predictive control (MPC) problem.
Chen et al.~\cite{chen2021interactive} develops a scenario-based MPC algorithm to handle multi-modal reactive behaviors of uncontrolled human agents.
In~\cite{hu2022sharp,hu2023activeIJRR}, a provably safe scenario-based MPC planner is proposed for interactive motion planning in uncertain, safety-critical environments, which improves task performance by preempting emergency safety maneuvers triggered by unlikely agent behaviors.
Recent work~\cite{li2023scenario} uses alternating direction method of multipliers (ADMM) based scenario optimization to tractably compute Nash equilibria for a class of constrained stochastic games subject to parametric uncertainty.

Maximum-entropy optimal control constitutes another branch of energy-based methods (essentially equivalent KL control with a uniform reference policy) that incentives entropy of the optimal policy, which can be used to incentivize policy robustness \cite{eysenbach2021maximum}, learn from human data \cite{wu2020efficient}, and escape local minima in long-horizon tasks \cite{pitis2020maximum}.
Mehr et al.~\cite{mehr2023maximum} take a step forward into the multi-agent land and study a maximum-entropy dynamic game. 

In this work, we provide a mechanism for incorporating policies guided by data-driven models from large aggregated datasets in model-based game-theoretic planning.
We show that our formulation is a generalization of the maximum-entropy dynamic game when the reference policy becomes arbitrarily uninformative.
We also leverage scenario optimization to tractably handle multi-modal reference policy at scale.

\subsection{Multi-Modal Multi-Agent Motion Prediction}
\label{subseq:multi_modal_motion_prediction}

Multi-agent trajectory prediction seeks to model agent-to-agent interactions that greatly affect joint outcomes versus predicting independently for each agent \cite{helbing1995social}. 
Recently, multi-agent motion prediction \cite{shi2022motion, shi2023mtr++, salzmann2020trajectron++} has enjoyed great success in predicting multi-modal human behavior in part due to advances in transformer-based architectures \cite{zhou2022hivt, jia2023hdgt}, natural language processing \cite{seff2023motionlm}, and diffusion models \cite{jiang2023motiondiffuser}. The output of a motion predictor is a weighted set of trajectory samples \cite{Ngiam2021SceneTransformer}, typically in the form of a mixture model over discrete modes \cite{Varadarajan2021MultipathPlusPlus}.
Finally, some approaches explicitly model discrete agent interactions in order to better account for them and improve accuracy
\cite{kumar2020interaction, sun2022m2i,ban2022deep}.

In sample-based prediction, maintaining the diversity of samples to represent distinct outcomes is an active area of research. Recent works \cite{lidard2023nashformer, huang2023gameformer} emphasize \textit{multi-modality} by incorporating rollouts and fictitious play to increase diversity in the prediction framework. Farthest Point Sampling (FPS) \cite{huang2020diversitygan,shiroshita2020behaviorally}, Non-Maximum Suppression (NMS) \cite{zhao2021tnt}, and/or neural adaptive sampling \cite{huang2022hyper} provide methods for encouraging outcome diversity using pairwise distances between trajectories as a metric. Sample diversity makes prediction methods a popular data-driven choice for human behavior. 

Prediction models have been incorporated in planning through an integrated prediction and planning approach, typically a model-predictive controller that seeks to avoid other agents \cite{lindemann2023safe, dixit2023adaptive}. In \cite{espinoza2022deep}, joint motion forecasts are incorporated as a constraint in the motion planning framework, allowing the planning agent to better understand high-dimensional scene data (such as the map) through low-dimensional action selection. In~\cite{huang2023differentiable}, a parameter-shared cost model is learned from data for all agents in the scene, allowing for scene-dependent cost formulation. 

\rebuttal{Generative modeling of traffic scenarios is also capable of compositing learned behavior behaviors with model-based objectives. Recent works \cite{zhong2023guided, zhong2023language} utilize signal temporal logic to guide the diffusion model generating desired behaviors. VBD \cite{huang2024versatile} applies a game-theoretic gradient descent-ascent guidance strategy to generate interactive safety-critical scenarios with a diffusion model. 
}

\rebuttal{In this work, we provide a principled approach to integrate the \textit{multi-modality} of motion predictors with closed-loop game-theoretic motion planning. Our framework complements end-to-end learned methods by allowing an agent to consider multiple motion modes and compute an optimal policy for each one, rather than needing to select a mode \textit{a-priori}. Furthermore, by specifying additional incentives such as safety, KLGame allows a robot to reason \textit{strategically} over time in a \textit{closed-loop} manner, making it more robust to unmodelled scenarios through feedback when compared to open-loop end-to-end predictions.}

\section{Problem Formulation}
\label{sec:problem_formulation}
\begin{figure*}[t!]
    \centering
    \includegraphics[width=\textwidth]{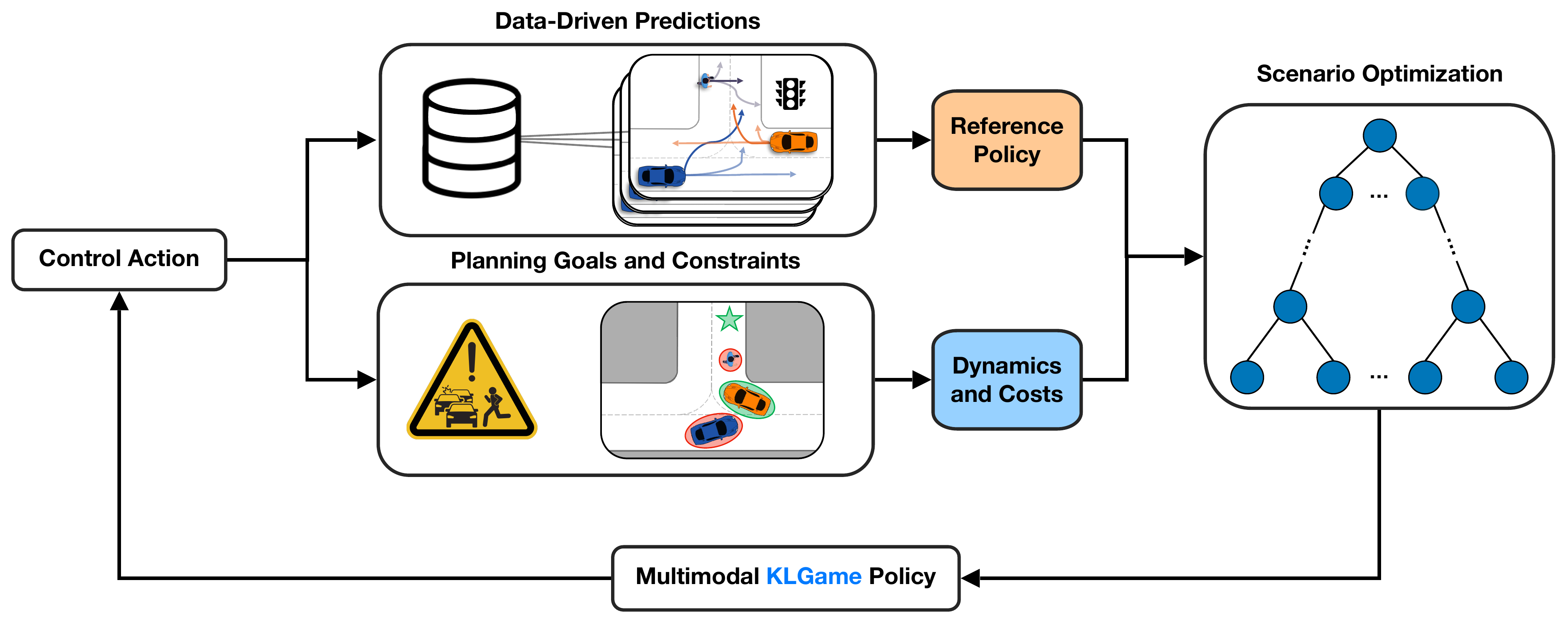}
    \caption{Overview of the proposed algorithmic framework for computing an approximate feedback Nash equilibrium for KLGame. In essence, we use scenario optimization to compute a multi-modal mixed strategy via blending a multi-modal reference policy learned from human interaction data with optimization-based game-theoretic policy.}
    \label{fig:system_diagram}
\end{figure*}

\noindent \textbf{Notation}. \rebuttal{We will take $T$ as the planning horizon and define $[T] := \{1,...,T\}$. Similarly, if there are $N$ players we will use $[N] := \{1,...,N\}$ to refer to the set containing all the players' indices. Throughout the paper we will use subscripts to refer to time and superscripts for player. For example, $u_t^i$ will stand for the action of player $i$ at time $t$. For a matrix $M$ we use $M \succ 0$ (or $M \succeq 0$) to indicate M is positive (or semi-positive) definite. We use $\sigma(M)$ to indicate the set of singular values of $M$.
}

We consider a discrete-time dynamic game with $N$ robots over a finite horizon $T$ governed by the nonlinear dynamics
\begin{equation} \label{eqn: dynamics}
    x_{t+1} = \dyn_t(\jxt, \jut),
\end{equation}
where $\jxt \in \mathcal X$ is the joint state of the dynamical system and $u_t := (u_t^1, ..., u_t^N) $
is the joint control input. Each player $i \in [N]$ applies an action according to a stochastic policy $\policy_t^i(\ctrl^i_t | \state_t)$ at each time step $t$. 


Each player $i$ minimizes a cost $J^i$ that quantifies the loss associated with the trajectory and control effort exerted by the player as well as the deviation with respect to a reference policy $\policyref_t^i(\ctrl^i_t | \state_t)$ over a time horizon $T$:
\begin{equation} \label{eqn: total_cost}    J^i(\pi^i) = \mathbb E^{ \pi }  \left[ \sum_{t=0}^T \; l^i(\jxt, \jut)+ \reg^i D_{KL}(\pi_t^{i}(\cdot | \jxt) || \tilde \pi_t^i(\cdot|\jxt) ) \right] ,
\end{equation}
where $\pi := (\pi^1, ..., \pi^N)$ is the joint policy of all robots, $\lambda$ is an $N$-vector of nonnegative scalars, and $ D_{KL}$ is the Kullback-Leibler (KL) divergence from the reference policy,
\begin{equation} \label{eqn: kl_definition}
    D_{KL}(\pi^{i}_t || \tilde \pi^i_t) = \int_{\mathcal U} \pi^i_t(u^i_t) \log \Bigg( \frac{\pi^i_t(u^i_t)}{\tilde \pi^i_t(u^i_t)} \Bigg) du^i_t.
\end{equation}
For the KL divergence to be well-defined, we require that $\policy_t^i(\cdot | \jxt)$ is absolutely continuous with respect to $\tilde \pi_t^i(\cdot|\jxt)$, i.e., \rebuttal{$\tilde \pi_t^i(\cdot|\jxt) = 0 \implies \policy_t^i(\cdot | \jxt) = 0$ \cite{ccinlar2011probability}. We denote absolute continuity as} $\policy^i_t << \tilde \pi_t^i$. 

The scalar $\reg^i \geq 0$ provides a natural ``tuning knob'' for the stochastic dynamic game, balancing the preference between optimizing the player's nominal performance objectives, encoded in $l^i$, and aligning its policy with the reference policy, $\policyref^i$.
We refer to this new class of dynamic game as \emph{KLGame}, which can be compactly formulated as:
\begin{equation} \label{eqn: overall optimization}
\begin{array}{rlrlcl}
\displaystyle \min_{\pi^i} & J^i(\pi^i)\\
\textnormal{s.t.} & x_{t+1} = \dyn_t(\jxt, \jut), & \forall t,
\end{array}
\end{equation}
for all $i \in [N]$. Herein, we use $\neg i$ to denote all agents except for $i$, and the superscript $(\cdot)^*$, as an indicator of optimality.


Recent work~\cite{li2023cost} revealed that feedback Nash equilibria (FBNE) strategies can generate more expressive and realistic interactions than their open-loop counterpart, which is an important property for interaction-rich and safety-critical motion planning tasks such as autonomous driving.
For the same reason, we focus on seeking FBNE of game~\eqref{eqn: overall optimization} in this paper.

\begin{definition}[Feedback Nash Equilibrium~\cite{bacsar1998dynamic}]
    A policy ${\policy}^{i*}$ defines an FBNE if no player has an incentive to unilaterally alter the strategy, i.e.,
\begin{equation*}
    J(\pi^{i*}, \pi^{\neg i*}) \leq J(\pi^i, \pi^{\neg i*}), \quad \forall \policy_t^i(\cdot) \in \Pi^i,~\forall t \in [T],~\forall i \in [N].
\end{equation*}
\end{definition}

We now introduce the main technical approach for seeking FBNE of KLGame.

\section{Kullback-Leibler Regularized Games}
\label{sec:KLGame}

In this section, we propose an efficient algorithm for seeking FBNE of KLGame \eqref{eqn: overall optimization}.
Our roadmap towards such a game solver is as follows.
First, we show that a KL-regularized Bellman equation adopts a closed-form solution.
Then, we derive the analytical global Nash equilibrium strategy in the LQ setting with Gaussian reference policies.
Finally, we propose an iterative linear-quadratic (ILQ) algorithm for efficiently approximating the game solution with general nonlinear dynamics, costs, and multi-modal reference policies
based on the ILQ approximation principle. Fig. \ref{fig:system_diagram} presents the high-level architecture of our system.

\subsection{Dynamic Programming Formulation}
\label{subsec:DPformulation}


We propose to solve \eqref{eqn: overall optimization} with the dynamic programming principle, which gives the Bellman equation:
\begin{equation} \label{eqn: bellman_KL}
\begin{aligned}
        \valfunc^{i}(\state) =\inf_{\pi^{i}} \Big \{&  \mathbb E^{\pi} [\qfunc^i(\state, \ctrl)] + \reg^i D_{KL}[\pi^i(\cdot|\state)||\tilde\pi^i(\cdot|\state)] \Big \},
\end{aligned}
\end{equation}
where $\valfunc^{i}(\state) = \inf_{\pi^{i}}  J^i(\pi^i)$ denotes the optimal
value function and $\qfunc^i(\state, \ctrl) := l^i(\state, \ctrl) + V^{i\prime}(f(\state, \ctrl))$ is the Q-value function.
In the remainder of this section, we show that \eqref{eqn: bellman_KL} admits an optimal solution of a particular structure.


We first state a helpful Lemma, which is a slight modification to the dual of the Donsker-Varadhan variational formula \cite{donsker1983asymptotic}.



\begin{lemma}
(Regularized variational formula~\cite{dupuis2011weak,dupuis2019representations}.)
\label{lem:Dupuis}
Let $\qfunc: \mathcal X \times \mathcal U \rightarrow \mathbb R$ be any bounded measurable function. Then, for any nonnegative real number $\reg$, probability measures $\pi$, $\tilde \pi$ and states $x \in \mathcal X$:
\begin{equation} \label{eqn: dupuis_lemma}
\begin{aligned}
        -\reg &\log \int_{\mathcal U} e^{-\qfunc(x, u)} \tilde{\pi}(u|x)du  \\
        &= \underset{\pi}{\inf} \Bigg\{ \expectation^\policy \left[ \qfunc(x,u) \right] + \reg D_{KL}(\pi(\cdot | x) || \tilde \pi(\cdot | x)) \Bigg\}.
\end{aligned}
\end{equation}
Moreover, the infimum in \eqref{eqn: dupuis_lemma} is attained with minimizer $\pi^*$ if $\pi^*$ satisfies the expression:
\begin{equation}
    \pi^*(u|x) = \frac{e^{-\qfunc(x,u)/\reg} \tilde \pi(u|x) }{\int_{\mathcal U} e^{-\qfunc(x, u)/\reg} \tilde \pi(u|x)du }.
\end{equation}
\end{lemma}
\begin{proof}
    See Appendix~\ref{apdx:proof_lem_Dupuis}.
\end{proof}
%
Lemma~\ref{lem:Dupuis} presents a convenient analytical form of the solution to Bellman equation~\eqref{eqn: bellman_KL}:
the optimal regularized policy is a (normalized) product of two density functions, one of which being precisely the reference policy.
This result, as we will see momentarily, serves as a key step towards deriving the KLGame equilibrium strategy.

\subsection{Global FBNE: The Linear-Quadratic Gaussian Case}
\label{subseq:globalFBNE_LQG}
Even without the KL regularizer in \eqref{eqn: total_cost}, seeking a local feedback Nash equilibrium of a non-LQ trajectory game is generally intractable~\cite{laine2023computation}. However, if the system dynamics are linear, the cost functions are convex quadratic functions of the state (LQ), and the reference policy is Gaussian, the \emph{global} Nash equilibrium can be computed efficiently.
\rebuttal{Specifically, the system dynamics are linear when
\begin{equation} \label{eqn: linear_dynamics}
    x_{t+1} = A_t\jxt + \sum_{i \in [N]} B_t^i u^i_t + \dstb_t,
\end{equation}
where $A_t,B_t^i$ are the state-transition matrices
and $\dstb_t \sim \gaussian(0, \covar_{\dstb})$ Gaussian noise with zero mean and covariance $\Sigma_\dstb$.
The stage cost of each player is a quadratic function of state and control, i.e., $l^i(x_t, u_t) = \jxt^\top Q^i_t \jxt + \sum_{j \in [N]} {u_t^j}^\top R_t^{ij} u_t^j$,
where $Q^i$ and $R^{ij}$ are positive semidefinite matrices of appropriate dimension penalizing state and control authority. 
The reference policies $\policyref^i$ are Gaussian if
\begin{equation} \label{eqn: gaussian_reference}
    \policyref^i_t = \gaussian(\tilde{\mu}^i_t, \tilde{\Sigma}^i_t)
\end{equation}
for all $t \in [T]$ and $i \in [N]$.}
       
We call KLGame, in this special case, the KL-regularized Linear-Quadratic Gaussian (KL-LQG) game.
\rebuttal{
%
Our goal is to find, for all agents $i$, the global FBNE strategy $\pi^{i*}$ of a KL-LQG game. We show in the following theorem that $\pi^{i*}$ is Gaussian with a \textit{state-dependent} mean $\mu^{i*} = \mu^{i*}(x)$, which depends linearly on $x$, and a \textit{state-independent} covariance $\Sigma^{i*}.$ Furthermore, we show that $\mu^{i*}$ and $\Sigma^{i*}$  are parameterized by the dynamics ($A$ and $B$), task costs ($Q$ and $R$), reference policy ($\tilde{\mu}$ and $\tilde{\Sigma}$), and value function parameters.
}


\begin{theorem}[Global FBNE of KL-LQG Game]
\label{thm:global_nash}
The $N$-player nonzero-sum KL-LQG dynamic game~\eqref{eqn: overall optimization} admits a unique global feedback Nash equilibrium solution if,
\begin{enumerate}
    \item  \rebuttal{The dynamics follow the functional form of \eqref{eqn: linear_dynamics},
    where $x_{0} \sim \gaussian(\mean_{x_{0}}, \covar_{x_{0}})$, $\dstb_t \sim \gaussian(0, \covar_{\dstb})$},

    \item The costs, \rebuttal{introduced in \eqref{eqn: total_cost}}, have the functional form \label{hip:1.2}
    \begin{equation*}
    \begin{aligned}
        J^i(\pi^i) =& \mathbb E^{ \pi } \left[ \sum_{t=0}^T \frac{1}{2} \Big(\jxt^\top Q^i_t \jxt +
        \sum_{j \in [N]} {u_t^j}^\top R_t^{ij} u_t^j \Big)\right]\\ &+ \sum_{t=0}^T  \reg^i D_{KL}(\pi_t^{i} || \tilde \pi_t^i ) ,
    \end{aligned}
    \end{equation*}
    where, for all $t \in [T]$, $Q^i_t \succeq 0$, $R_t^{ij} \succeq 0$, $\forall i,j \in [N], j \neq i$, $R_t^{ii} \succ 0$, $\forall i \in [N]$,

    \item \rebuttal{The reference policies are Gaussian, as in \eqref{eqn: gaussian_reference}}.
\end{enumerate}
Moreover, the global (mixed-strategy) Nash equilibrium is a set of time-varying policies $\pi_t^{i*} = \mathcal N(\mu_t^{i*}, \Sigma_t^{i*}),~\forall i \in [N]$ with mean and covariance given by
\begin{equation} \label{eqn: kl_optimal_policy}
\begin{aligned}
    \mu^{i*}_t &=  - K^i_t \state_t - \kappa^i_t,   \\
    \Sigma^{i*}_t &= \left[ \textcolor{red}{\frac{1}{\reg^i}} \left(R^{ii}_t + {B^i_t}^T Z^i_{t+1} B^i_t \right) \textcolor{red}{+ \left(\tilde{\covar}^i_t\right)^{-1}} \right]^{-1},
\end{aligned}
\end{equation}
where $(K^i_t, \kappa^i_t)$, \rebuttal{the state-feedback matrix and constant}, are given by solving the coupled KL-regularized Riccati equation:
\begin{equation}
    \begin{aligned}
        \big[{\color{red} \reg^i (\Tilde{\Sigma}^i_t)^{-1}}+ R_t^{ii} + B_t^{i\top}\qterm &B_t^{i\top} \big] K^i_t +  B_t^{i\top}\qterm\sum_{j\neq i}B_t^j K^j_t \\ &= B_t^{i\top}\qterm A_t, \\
        \big[ {\color{red}\reg^i(\Tilde{\Sigma}^i_t)^{-1}} + R_t^{ii} + B_t^{i\top}\qterm &B_t^{i\top} \big]\kappa^i_t + B_t^{i\top}\qterm\sum_{j\neq i}B_t^j \kappa^j_t \\ &= B_t^{i\top}\lterm {\color{red} - \reg^i(\Tilde{\Sigma}^i_t)^{-1}\Tilde{\mu}^i_t}.
    \end{aligned}
    \end{equation}
Value function parameters $(Z^i_{t}, z^i_{t})$ are computed recursively backward in time as
\begin{equation}
\begin{aligned}
\label{eq:update}
    Z_t^i =& Q_t^i+\sum_{j\in[N]} (K_t^{j})^T R_t^{i j} K_t^j+F_t^T Z_{t+1}^i F_t {\color{red}+ \reg^i K^{i\top}_t(\Tilde{\Sigma}^i_t)^{-1}K^i_t},\\
    z_t^i =& \sum_{j\in[N]} (K_t^{j})^T R_t^{i j} \kappa^j_t+F_t^T\left(z_{t+1}^i+Z_{t+1}^i \beta_t\right) \\
    &{\color{red}+ \reg^i K^{i\top}_t(\Tilde{\Sigma}^i_t)^{-1}(\kappa^i_t - \Tilde{\mu}^i_t)},
\end{aligned}
\end{equation}
with terminal conditions $Z^i_{T+1}=0$ and $z^i_{T+1}=0$.
Here, $F_t = A_t-\sum_{j\in[N]} B_t^j K_t^j$ and $\beta_t=-\sum_{j\in [N]} B_t^j \kappa_t^j$ for all $t \in [T]$.
\end{theorem}
\begin{proof}
    See Appendix~\ref{apdx:thm_global_nash}.
\end{proof}

Theorem~\ref{thm:global_nash} reveals several important insights about KL-LQG games at equilibrium:
\begin{itemize}
    \item The global FBNE must be a mixed strategy, taking the form of an unimodal Gaussian.

    \item The policy mean is a linear state-feedback control law, coinciding with the function form of FBNE of a deterministic LQ game~\cite{bacsar1998dynamic}.
    Moreover, computing the global FBNE strategy enjoys the same order of complexity as the deterministic LQ game, since the coupled Riccati equation has the same dimension with additional \textcolor{red}{red terms} contributed by the reference policy,

    \item As the reference policy becomes arbitrarily uninformative, the FBNE of KL-LQG reduces to that of a maximum-entropy game~\cite{mehr2023maximum},

    \item As $\reg^i \rightarrow \infty$, the FBNE becomes the reference policy,

    \item As $\reg^i \rightarrow 0$, the FBNE coincides with that of a deterministic LQ game.
\end{itemize}

\begin{corollary}[Reduction to the MaxEnt Game]
As the reference policies become arbitrarily uninformative, i.e., $\min~ \sigma(\tilde{\covar}^i_t) \rightarrow \infty,\forall i \in [N],\forall t \in [T]$,
a KL-LQG game reduces to the maximum-entropy dynamic game defined in~\cite{mehr2023maximum} and its FBNE becomes an unguided, exploratory entropic cost equilibrium (ECE) strategy. 
\end{corollary}

\begin{proof}
    Recall from~\cite{mehr2023maximum} that an ECE in the LQ case is also an unimodal Gaussian.
    When $\min~\sigma(\tilde{\covar}^i_t) \rightarrow \infty$, the contribution from the reference policy (red terms) in the covariance $\covar^{i*}_t$ and Riccati equation vanishes, and the resulting covariance and Riccati equation match that of the maximum-entropy dynamic game.
\end{proof}

\begin{corollary}[Reduction to the Reference Policy]
As $\reg^i \rightarrow \infty$, the FBNE of a KL-LQG game becomes the reference policy, i.e., $\lim_{\reg^i \rightarrow \infty} \mu_t^{i *} = \tilde{\mu}_t^i$ and $\lim_{\reg^i \rightarrow \infty} \Sigma_t^{i *} = \tilde{\Sigma}_t^i$.
\end{corollary}

\begin{corollary}[Reduction to the Deterministic LQ Game]
As $\reg^i \rightarrow 0$, the FBNE of a KL-LQG game becomes a deterministic policy, i.e., $\lim_{\reg^i \rightarrow 0} \mu_t^{i *} = -K_t^i x_t-\kappa_t^i$, where $(K_t^i, \kappa_t^i)$ is given by the coupled Riccati equation of the deterministic LQ game.
\end{corollary}



While Theorem~\ref{thm:global_nash} lays the theoretical foundation of KLGame, the assumption that the reference policy $\policyref_t^i \sim \mathcal{N}(\tilde{\mu}_t^i, \tilde{\Sigma}_t^i)$ implies that the guidance is ultimately \emph{open-loop}: the moments of $\policyref_t^i$ are independent of the state and thus unaware of players' interaction throughout the planning horizon.
In the following, we extend Theorem~\ref{thm:global_nash} to the \emph{closed-loop} setting by accounting for a \emph{state-feedback} reference policy.

\begin{proposition}[Closed-loop Guidance]
\label{prop:cl_guide}
    Let all assumptions in Theorem~\ref{thm:global_nash} hold but assume reference policies are Gaussian conditioned on the time and state, i.e., $\policyref^i_t (\cdot | \state_t) \sim \gaussian(\tilde{\mu}^i_t(\state_t), \tilde{\Sigma}^i_t)$, where $\tilde{\mu}^i_t(\state) = -\policyrefGain^i_tx - \policyrefOffset^i_t$, for all $t \in [T]$ and $i \in [N]$.
    Under these assumptions, the global (mixed-strategy) Nash equilibrium is a set of time-varying policies $\pi_t^{i*} = \mathcal N(\mu_t^{i*}, \Sigma_t^{i*}),~\forall i \in [N]$ with mean and covariance given by
    \begin{equation} \label{eqn: cl_kl_optimal_policy}
    \begin{aligned}
        \mu^{i*}_t &=  - K^i_t \state_t - \kappa^i_t,   \\
        \Sigma^{i*}_t &= \left[ \textcolor{black}{\frac{1}{\reg^i}} \left(R^{ii}_t + {B^i_t}^T Z^i_{t+1} B^i_t \right) + \textcolor{black}{ \left(\tilde{\covar}^i_t\right)^{-1}} \right]^{-1},
    \end{aligned}
    \end{equation}
    where $(K^i_t, \kappa^i_t)$ is given by solving the coupled KL-regularized Riccati equation:
\begin{equation} \label{eqn: cl_kl_Riccati}
    \begin{aligned}
        \big[ \textcolor{black}{\reg^i (\Tilde{\Sigma}^i)^{-1}} + R_t^{ii} + B^{i\top}\qterm &B^{i\top} \big] K^i_t +  B^{i\top}\qterm\sum_{j\neq i}B^j K^j_t \\ &= B^{i\top}\qterm A {\color{orange} + \reg^i(\Tilde{\Sigma}^i_t)^{-1}\Tilde{K}^i_t},\\
        \big[ \textcolor{black}{\reg^i(\Tilde{\Sigma}^i)^{-1}} + R_t^{ii} + B^{i\top}\qterm &B^{i\top} \big]\kappa^i_t + B^{i\top}\qterm\sum_{j\neq i}B^j \kappa^j_t \\ &= B^{i\top}\lterm {\color{orange} + \reg^i(\Tilde{\Sigma}^i_t)^{-1}\Tilde{\kappa}^i_t}.
    \end{aligned}
    \end{equation}
    Value function parameters $(Z^i_{t}, z^i_{t})$ are computed recursively backward in time as
    \begin{equation}
    \begin{aligned}
    \label{eq:cl_update}
        Z_t^i =& Q_t^i+\sum_{j\in[N]}(K_t^{j})^T R_t^{i j} K_t^j+F_t^T Z_{t+1}^i F_t \\ 
        &{\color{orange}+  \reg^i (K^{i}_t + \Tilde{K}^{i}_t )^\top  (\Tilde{\Sigma}^i_t)^{-1}(K^{i}_t + \Tilde{K}^{i}_t )}, \\
        z_t^i =& \sum_{j\in[N]} (K_t^{j})^T  R_t^{i j} \kappa^j_t+F_t^T\left(z_{t+1}^i+Z_{t+1}^i \beta_t\right) 
        \\ &\textcolor{orange}{+ \reg^i (K^{i}_t + \Tilde{K}^{i}_t )^\top
        (\Tilde{\Sigma}^i_t)^{-1}(\kappa^{i}_t + \Tilde{\kappa}^{i}_t )},
    \end{aligned}
    \end{equation}
    with terminal conditions $Z^i_{T+1}=0$ and $z^i_{T+1}=0$.
    Here, $F_t = A_t-\sum_{j\in[N]} B_t^j K_t^j$ and $\beta_t=-\sum_{j\in [N]} B_t^j \kappa_t^j$ for all $t \in [T]$.
\end{proposition}


\begin{proof}
    See Appendix~\ref{apdx:prop_cl_guide}.
\end{proof}

We highlight in~\eqref{eqn: cl_kl_Riccati} and~\eqref{eq:cl_update} terms that are different from Theorem~\ref{thm:global_nash} \textcolor{orange}{in orange}, which stem from the state-feedback reference policy.


\subsection{Solving General KLGame Planning Problems}
\label{subsec:solving_general_KLGame}

Having gained insights into the theoretical properties of KLGame in the LQ setting with Theorem~\ref{thm:global_nash} and Proposition~\ref{prop:cl_guide}, we now return to the full KLGame setting with nonlinear dynamics~\eqref{eqn: dynamics}, non-convex costs~\eqref{eqn: total_cost}, and arbitrary (possibly state-dependent) reference policies $\policyref^i(\ctrl | \state)$.
The intractability of finding an exact FBNE in deterministic games~\cite{laine2023computation} does not ease in our setting.
Therefore, inspired by~\cite{fridovich2020efficient}, we instead look for an approximate FBNE with the linear-quadratic-Laplace (LQL) approximation method, an extension of the ILQGame method in~\cite{fridovich2020efficient} to the KLGame setting, which we introduce in detail below.

\subsubsection{The LQL approximation} \label{par: LQL}
We seek to find an approximate FBNE by iteratively solving a sequence of KL-LQG games until converging to a stationary point.
This process is initialized with a \textit{nominal} trajectory $\nominaltraj := \{\bar{\state}_{[0:T]}, \bar{\ctrl}_{[0:T]}\}$, and is divided into two phases: backward pass and forward pass.

\noindent 
\textbf{Backward Pass.}
During the backward pass, we linearize dynamics~\eqref{eqn: dynamics} and quadraticize cost~\eqref{eqn: total_cost} around $\nominaltraj$ to obtain a linear dynamical system and quadratic cost functions that fit into the assumption of KL-LQG game (c.f. Theorem~\ref{thm:global_nash}).
Next, we use the Laplace approximation~\cite[Chapter~4]{bishop2006PRML} to obtain a Gaussian distribution that captures the reference policy locally around $\nominaltraj$.
Specifically, the conditional probability distribution of a reference policy $\policyref^i_t(\ctrl^i_t | \state_t)$ is approximated as:
\begin{equation}
\label{eq:Laplace}
    \policyref^i_t(\ctrl^i_t | \state_t)
    \approx \gaussian\left( \tilde{\mu}^i_t(\state_t; \bar{\nominaltraj}_{t}), \tilde{\Sigma}^i_t(\bar{\nominaltraj}_{t}) \right),
\end{equation}
where $\bar{\nominaltraj}_{t} := \{\state_t, \ctrl_t\}$, the mean function is
\begin{equation}
\label{eq:Laplace:mean}
    \tilde{\mu}^i_t(\state_t; \bar{\nominaltraj}_{t}) := \textstyle\argmax_{\tilde{\ctrl}^i(\state_t)} \policyref^i_t(\tilde{\ctrl}^i | \bar{\state}_t),
\end{equation}
and the covariance matrix is
\begin{equation}
\label{eq:Laplace:covar}
    \tilde{\Sigma}^i_t(\bar{\nominaltraj}_{t}) := \textstyle - \left[ \nabla^2_{\ctrl^i} \ln \left.  \policyref^i_t(\ctrl^i | \state_t) \right|_{\ctrl^i = \tilde{\mu}^i_t(\state_t; \bar{\nominaltraj}_{t})} \right]^{-1}.
\end{equation}
The Laplace-approximated distribution in~\eqref{eq:Laplace} is a Gaussian whose mean is centered at a mode $\tilde{\mu}^i_t(\cdot)$ of the reference policy $\policyref^i_t(\cdot)$.
As discussed, it is desirable to use closed-loop guidance with a state-dependent reference policy (c.f. Proposition~\ref{prop:cl_guide}).
We note that, in this case, \emph{functional optimization} is required in order to find the state-dependent mean function $\tilde{\mu}^i_t(\state_t; \bar{\nominaltraj}_{t})$ in~\eqref{eq:Laplace:mean}.
While an exact optimization is likely intractable, we provide a practical implementation to find such a function.
First, we sample for $t \in [T]$ the original reference policy $\policyref^i_t(\ctrl^i_t | \bar{\state}_t)$ conditioned on the current nominal trajectory $\nominaltraj$ and compute the sample means (control actions).
Then, we solve an ILQR problem, which tracks those controls, for a time-varying linear state-feedback control law.
This control law may be used as the mean function in~\eqref{eq:Laplace:mean}, and the covariance matrix in~\eqref{eq:Laplace:covar} can be computed by setting $\ctrl^i$ to those sample means.

With the linearized dynamics, quadraticize costs, and Laplace-approximated reference policy obtained above, we can now formulate a KL-LQG game, whose global FBNE can be found by leveraging Theorem~\ref{thm:global_nash} (or Proposition~\ref{prop:cl_guide}).

\noindent 
\textbf{Forward Pass.} 
The forward pass aims to update the nominal trajectory $\nominaltraj$ with the KL-LQG strategies.
This can be done by forward-simulating the system with players' control deviation at each time set to the KL-LQG strategy mean $\mu^{i,*}_t(\state_t; \bar{\nominaltraj}_{t})$, i.e., $\ctrl^i_t = \bar{\ctrl}^i_t + \delta \ctrl^i_t,~\delta \ctrl^i_t = \mu^{i,*}_t(\state_t; \bar{\nominaltraj}_{t})$ for all $i \in [N]$ and $t \in [T]$.


\noindent 
\textbf{Line Search.}
A common issue in ILQ-based trajectory optimization is that directly applying the mean of the KL-LQG strategy in the forward pass may cause the system to deviate too much from the nominal trajectory $\nominaltraj$, resulting in divergence of the algorithm.
A possible remedy is line search~\cite[Chapter~3]{nocedal2006numerical}, which controls the trajectory update by scaling down the KL-LQG strategy parameters.
Specifically, for a line search parameter $\varepsilon \in (0, 1]$ and KL-LQG strategy $\mathcal N(- K^i_t \state_t - \kappa^i_t, \Sigma_t^i)$, the adjusted strategy is $\mathcal N(- \varepsilon K^i_t \state_t - \varepsilon \kappa^i_t, \varepsilon \Sigma_t^i)$, which is then used in the forward pass to obtain a new nominal trajectory $\nominaltraj$.
The procedure starts with $\varepsilon \gets 1$ and repeats by decreasing $\varepsilon$ by half until a convergence criterion is met.
Two commonly used criteria are decreasing of the social cost (i.e., the sum of all players' costs) and that the updated nominal trajectory is sufficiently close to the previous one. 

\subsubsection{Accounting for multi-modality with scenario optimization}
\label{sec:mmklg}
The LQL approximation may be extended to account for both \emph{multi-modal} approximate FBNE KLGame policy \emph{and} reference policy with $M$ discrete modes:
\begin{equation}
\label{eq:ref_policy_mm}
\policyref^i_t(\ctrl^i_t | \state_t) = \sum_{\mode = 1}^{M} \mixturecoeff^{i, \mode}_t \policyref^{i,\mode}_t(\ctrl^i_t | \state_t),
\end{equation}
where \textit{mixture weights} $\mixturecoeff^{i, \mode}_t \geq 0,~\sum_{\mode = 1}^{M} \mixturecoeff^{i, \mode}_t=1$, and \textit{mixture components} $\policyref^{i,\mode}_t(\ctrl^i_t | \state_t)$ are arbitrary probability density functions.
Due to the need for capturing multi-modality in the KLGame and reference policy, it no longer suffices to perform forward and backward passes along a single nominal trajectory $\nominaltraj$.
In order to tractably compute a multi-modal approximate FBNE of KLGame in this setting, we propose to use scenario optimization (SO)~\cite{bernardini2011stabilizing, campi2018general} as the underlying solution framework.
SO performs trajectory optimization over a scenario tree, whose root node is the initial state $\state_0$, which then descends into $\underline{M}$ nodes ($\underline{M} \leq M$), each corresponding to a different mode, either specified by the user or determined at random.
Each branch of the tree is a \emph{scenario} governed by a sequence of reference policy modes.
While most LQL approximation concepts carry over,
we discuss necessary modifications to the backward and forward pass procedures under the SO framework below.


\noindent
\textbf{Modified Backward Pass.}
First, we need to change the Laplace approximation to account for the multi-modality of the reference policy $\policyref^i_t(\ctrl^i_t | \state_t)$.
We propose to approximate the reference policy as a \emph{Gaussian mixture model (GMM)}, i.e., $\policyref^i_t(\ctrl^i_t | \state_t) \approx \sum_{\mode = 1}^{M} \mixturecoeff^{i, \mode}_t \gaussian\left( \tilde{\mu}^{i,\mode}_t(\state_t; \bar{\nominaltraj}_{t}), \tilde{\Sigma}^{i,\mode}_t(\bar{\nominaltraj}_{t}) \right)$.
This is obtained by inheriting the mixture coefficients $\mixturecoeff^{i, \mode}_t$ from the original reference policy~\eqref{eq:ref_policy_mm} and applying Laplace approximation to each mixture component $\policyref^{i,\mode}_t(\cdot)$.

Next, we modify the backward-time computation of the KL-LQG strategies so that it is performed over the scenario tree, thus preserving the time causality of the resulting strategies~\cite{mesbah2016stochastic}.
Specifically, given a scenario tree with node set $\nodeset$, for each node $\node \in \nodeset$ with time step $t$, the value function parameters of the succeeding time step in Theorem~\ref{thm:global_nash} (or Prop.~\ref{prop:cl_guide}) are obtained as a weighted sum of children node solutions, i.e.,
$Z^{i,\node}_{t+1} = \sum_{\tnode \in \Phi(\node)} \mixturecoeff^{i,\tnode}_{t+1} Z^{i,\tnode}_{t+1}$ and $z^{i,\node}_{t+1} = \sum_{\tnode \in \Phi(\node)} \mixturecoeff^{i,\tnode}_{t+1} z^{i,\tnode}_{t+1}$, where $\Phi(\node)$ is the set of children nodes of $\node$.

Finally, we propose to obtain a \emph{GMM} approximate FBNE policy of the KLGame for each non-leaf node $\node$ in the scenario tree:
$$\policy^{i,n,*}_t (\ctrl^i_t | \state_t) = \sum_{\tnode \in \Phi(\node)} \mixturecoeff^{i, \tnode}_t \gaussian\left( \tilde{\mu}^{i,\tnode}_t(\state_t; \bar{\nominaltraj}_{t}), \tilde{\Sigma}^{i,\tnode}_t(\bar{\nominaltraj}_{t}) \right).$$
In this GMM policy, the mixture coefficients $\mixturecoeff^{i, \tnode}_t$ are set to the values used for tree construction, and the mean $\tilde{\mu}^{i,\tnode}_t(\state_t; \bar{\nominaltraj}_{t})$ and the covariance $\tilde{\Sigma}^{i,\tnode}_t(\bar{\nominaltraj}_{t})$ are computed with~\eqref{eqn: kl_optimal_policy} (or~\eqref{eqn: cl_kl_optimal_policy}) by setting the succeeding value function parameters to that of the children node, i.e, $Z^i_{t+1} = Z^{i,\tnode,*}_{t+1}$ and $z^i_{t+1} = z^{i,\tnode,*}_{t+1}$.

\noindent
\textbf{Modified Forward Pass.}
The forward pass computation is now performed over the scenario tree: at a node $\node$ with nominal state $\bar{\state}^\node_t$, nominal control $\bar{\ctrl}^{i,\node}_t$, and KLGame policy $\policy^{i,n,*}_t (\ctrl^i_t | \state_t)$, for each mode $\mode$ of $\policy^{i,n,*}_t(\cdot)$, we compute a control action $\ctrl^{i,m}_t = \bar{\ctrl}^{i,\node}_t + \tilde{\mu}^{i,\mode}_t(\state_t; \bar{\nominaltraj}_{t})$.
Those control actions bring the system to the $\underline{M}$ children nodes of $\node$, and the procedure continues until the nominal state and control are updated for all nodes in the scenario tree.

\noindent
\textbf{The Algorithm.} We summarize the procedure of computing a GMM approximate FBNE of multi-modal KLGame in Algorithm~\ref{alg:cap}. We omit the algorithm for the basic KLGame introduced in the previous section since it is a special case of the multi-modal one.
In practice, we apply the KLGame policy in receding horizon fashion: each player $i$ executes action $\ctrl^i = \bar{\ctrl}^{i,r} + \delta\ctrl^{i,r}$, where $\bar{\ctrl}^{i,r}$ is the nominal control of the root node $r$ and $\delta\ctrl^{i,r}$ is sampled from the GMM policy $\policy^{i,r,*} (\ctrl^i | \state)$ of the root node, time horizon is shifted, and KLGame is solved repeatedly once a new state measurement is received.

\noindent \rebuttal{
\textbf{Computation Complexity.}
The computation complexity of multi-modal KLGame is affected by the state dimension $n$, each player’s control dimension $m_i$, the number of players $N$, the number of scenario tree branches $H$, and the number of LQL iterations (c.f. Section \ref{par: LQL}). Assuming $n > m_i$ for each player $i$, in each iteration, along a single tree branch, the complexity of the LQ approximation and solving the coupled Riccati equation of the corresponding KL-LQG game (c.f.~\cite[Corollary 6.1]{bacsar1998dynamic}) are $\mathcal{O}(Nn^2)$ and $\mathcal{O}(N^3n^3)$, respectively, coinciding with that of ILQGame~\cite{fridovich2020efficient}. The Laplace approximation has complexity $\mathcal{O}(Nn^3)$. The computation bottleneck of Algorithm~\ref{alg:cap} is the matrix inversion operation when solving the coupled Riccati equation, which has complexity $\mathcal{O}(N^3n^3)$. Therefore, the overall per-iteration complexity of multi-modal KLGame is $\mathcal{O}(HN^3n^3)$ (for ILQGame, the complexity is $\mathcal{O}(N^3n^3)$). Our implementation leverages the automatic vectorization of JAX~\cite{jax2018github} to efficiently batch-compute gradients and Hessians across different players and time steps. Throughout our experiments, we observe that Algorithm~\ref{alg:cap}, when warmstarted and solved in a receding horizon fashion, typically converges with a small number of LQL iterations.
As shown in Fig.~\ref{fig:runtime}, our solver achieves a 10Hz planning rate for a reasonable number of agents (around four) in complex traffic interaction scenarios where agents are in close proximity.
}

\begin{algorithm}[t!]
\caption{Multi-modal KLGame}\label{alg:cap}
\begin{algorithmic}
\Require Initial state $x_0$, $M$-mode reference policy $\policyref^i_t(\ctrl^i_t | \state_t)$, branching number $\underline{M} \leq M$

\State \textbf{Initialization}: Construct a scenario tree with node set $\nodeset$ and nominal trajectories $\eta_{\nodeset} := \{\state^\node, \ctrl^{i, \node}\}_{\node \in \nodeset}$


\While{\texttt{Not converged}} 
\State $\{\policy^{i,n}_t (\cdot)\}_{\node \in \nodeset} \gets$\textsc{BackwardPass}($\state_0$, $\eta_{\nodeset}$, $\policyref^i_t(\cdot)$)
\State $\{\policy^{i,n}_t (\cdot)\}_{\node \in \nodeset}$, $\eta_{\nodeset}$ $ \gets$\textsc{ForwardPass\&LineSearch}(\newline$\state_0$, $\eta_{\nodeset}$, $\{\policy^{i,n}_t (\cdot)\}_{\node \in \nodeset}$)
\EndWhile

\Ensure GMM approximate FBNE strategies $\{\policy^{i,n,*}_t (\cdot)\}_{\node \in \nodeset}$ and optimized nominal trajectories $\{\state^{\node,*}, \ctrl^{i, \node, *}\}_{\node \in \nodeset}$
 
\end{algorithmic}
\end{algorithm}

\section{Experimental Results}
\label{sec:experiments}
\newlength{\tempdima}
\newcommand{\colname}[1]
{{\makebox[\tempdima][c]{\textbf{#1}}}}
\newcommand{\rowname}[1]
{\rotatebox{90}{\makebox[\tempdima][c]{\textbf{#1}}}}

In this section, we study the role of the reference policy in helping KLGame find diverse game solutions more amenable to execution than existing methods. We compare our method against a mixture of deterministic, stochastic, and data-driven baselines on three simulated interaction scenarios with nonconvex costs and mode uncertainty. In all simulations, we use the 4D kinematic bicycle model~\cite{zhang2020optimization} to describe the motion of the vehicles.
We implement a receding horizon version of multi-modal KLGame using JAX~\cite{jax2018github}, which runs at 10Hz on a desktop with an AMD Ryzen 9 7950X CPU.

\subsection{Two-Player Tollbooth Selection}
\label{subsec:twoplayertollbooth}

\begin{figure}[t!]
\centering
\settoheight{\tempdima}{\includegraphics[width=.35\textwidth]{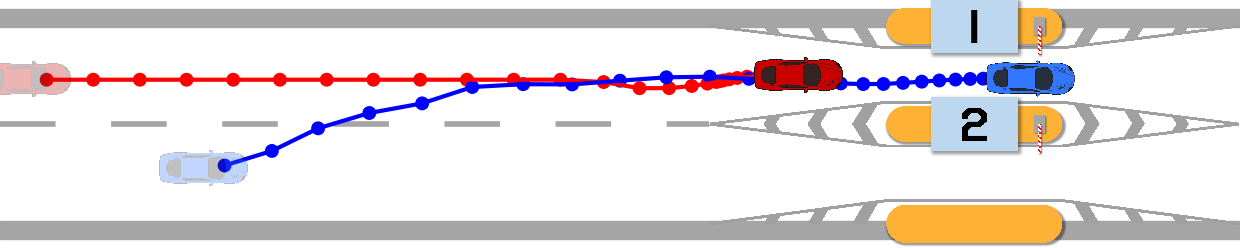}}
\begin{tabular}{@{}c@{}c}
& \colname{}  \\
\rowname{ILQGames} & \includegraphics[width=0.95\linewidth]{img/tollbooth/ilqgame_full_rollout.png}  \\
\rowname{MaxEnt} & \includegraphics[width=0.95\linewidth]{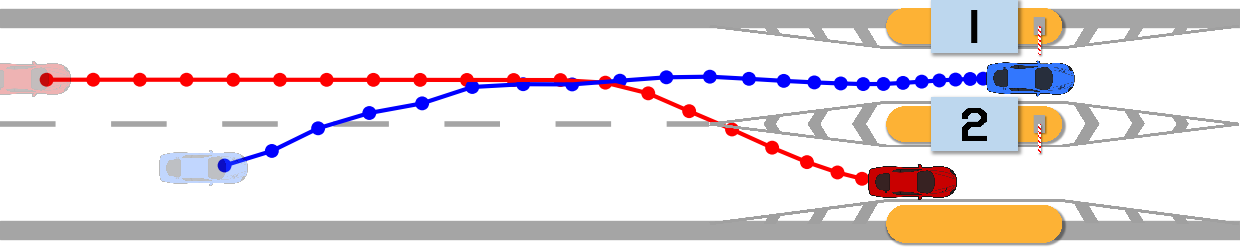}  \\
\rowname{KLGame} & \includegraphics[width=0.95\linewidth]{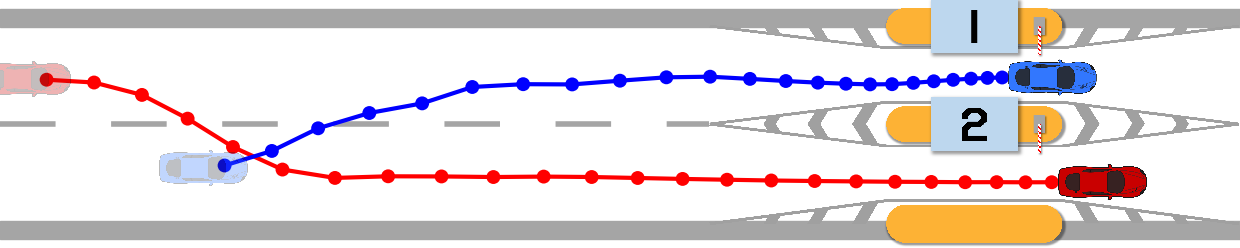} \\
\end{tabular}
\caption{Two-player tollbooth interaction. Player 1 (red) and Player 2 (blue) are approaching a tollbooth on a two-lane road (upper: 1, lower: 2). Both players receive a coordination bonus for taking opposite lanes (e.g. for reducing traffic), but Player 2 has a strong preference to merge into Lane 1. Since players incur a large penalty for deviating from the lane centerline, Player 1 may get caught in a local cost minimum (Lane 1) without external guidance to switch to Lane 2. }
\label{fig:tollbooth}
\end{figure}

In this experiment, we use the highway tollbooth coordination example from~\cite{hu2023emergent} to showcase the KLGame algorithm's ability to incorporate data-driven exploration in order to escape local minima despite nonconvex costs. We introduce two baseline methods.

\noindent\textbf{ILQGames \cite{fridovich2020efficient}.} ILQGames is a popular algorithm for solving near-LNE policies via a set of coupled algebraic Riccati equations. While optimizing for socially optimal costs allows players to find mutually beneficial trajectories, nonconvexity in costs coupled with a lack of exploration in the state space may incentivize the joint system to remain stuck at a local minimum. 


\noindent\textbf{Maximum Entropy Game \cite{so2022maximum, so2023mpogames}.} Inspired by the success of the maximum-entropy framework in games, we explore how random exploration can be used to escape local minima when costs are nonconvex. In maximum-entropy games, the optimal exploration rate is proportional to the Hessian of the control cost \cite{so2022maximum}. In contrast, the optimal KLGame policy incorporates the mean and covariance of a reference policy, allowing better exploration of focal regions of the state space, such as right or left turns. 


\begin{figure}[t!]
    \centering
    \includegraphics[width=0.3\textwidth]{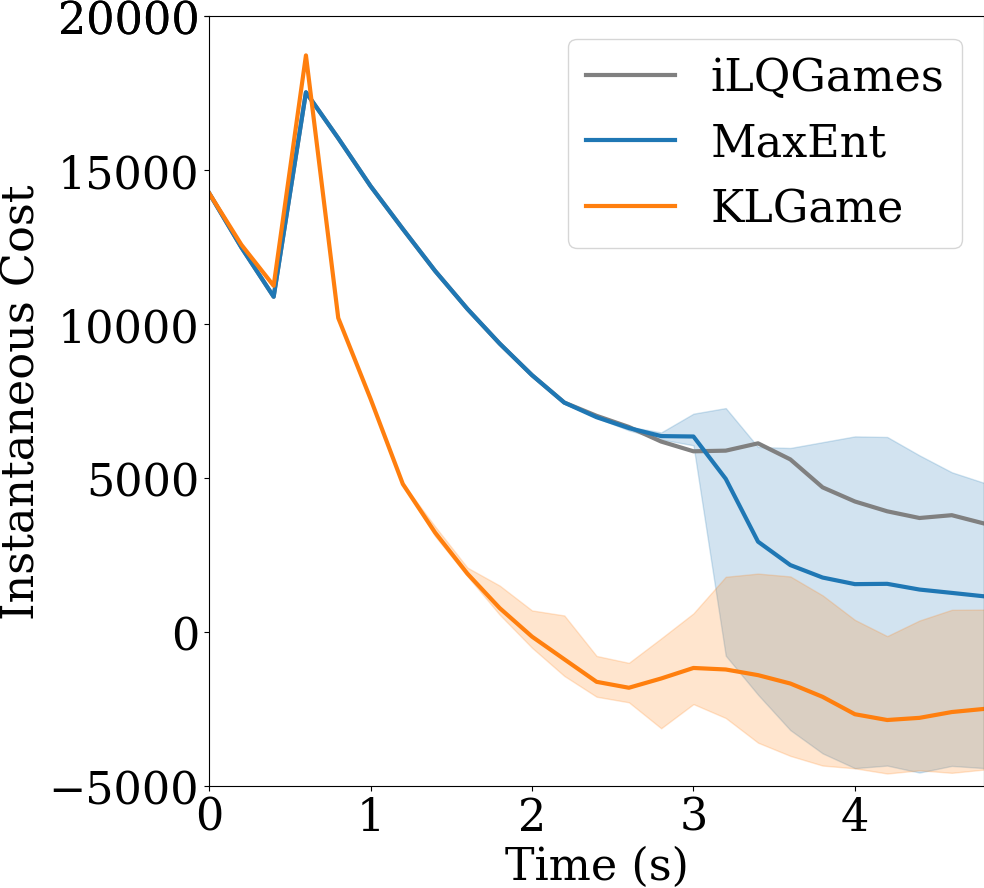}
    \caption{Instantaneous task costs for ILQGames, MaxEnt, and KLGame methods over $100$ trials. The initial spike in cost is due to Player 2 merging into Player 1's lane, eliminating the coordination bonus. MaxEnt incorporates stochastic exploration to find lower-cost trajectories, but incurs high variance. KLGame uses guidance from the reference policy to find a coordinating trajectory at an earlier iteration than MaxEnt, recovering the coordination bonus sooner and giving a lower overall cost.}
    \label{fig: costs}
\end{figure}

\setlength\tabcolsep{4pt} 
\begin{table}[t!]
\rebuttal{
    \centering
    \caption{Effect of Regularization in Tollbooth Interaction}
    \label{tab: klgame_baseline_comparison}
    \begin{tabular}{l|cccc}
        Method & CR $\uparrow$ & SR$\uparrow$ &  Prog. (m) $\uparrow$  &   Cost ($10^3$)$\downarrow$  \\
        \hline \hline
        ILQGames & $0.00\pm0.00$ & $\mathbf{1.00}\pm0.00$ & $45.0\pm0.00$ &  $8.45\pm0.00$ \\
        MaxEnt & $0.28\pm0.43$ & $0.68\pm0.48$ & $46.8\pm3.03$ &  $6.91\pm1.64$ \\
        \hline
        \textit{KLGame (Ours)} & $\mathbf{1.00}\pm0.00$ & $\mathbf{1.00}\pm0.00$ & $\mathbf{51.8}\pm5.83$  & $\textbf{4.22}\pm1.12$
    \end{tabular}
}
\end{table}

Fig. \ref{fig:tollbooth} compares qualitative behavior from KLGame against two methods that incorporate different regularization (no regularization for ILQGames and entropy regularization for MaxEnt). Two cars are rapidly approaching two toll stations and wish to get through the stations as quickly as possible and without slowing down. To capture the representative for each method, we show the \textit{ mean clustered equilibrium }for the three different methods. The ILQGames Player 1 is fully deterministic, and Player 1 must break and wait for Player 2 instead of changing lanes. The MaxEnt game Player 1 is able to explore and find the coordinating plan, but such exploration results in a delayed switch, and moreover, a penalty for crossing the lane boundary. The KLGame Player 1 incorporates a reference policy to find the coordinating plan. Using only a constant turn rate reference as a proxy for interactive toll-booth data, the KLGame Player 1 merges early and incurs the largest coordination bonus. Fig.~\ref{fig: costs} depicts instantaneous costs for the three methods, with KLGame incurring the lowest instantaneous and cumulative cost.

Table~\ref{tab: klgame_baseline_comparison} presents performance comparisons for the three methods over $100$ trials. We report the following methods: (i) coordination rate (CR): the proportion of trials that end with the vehicles in opposite lanes, (ii) safety rate (SR): the proportion of trials where the vehicles remain in-bounds and do not collide, (iii) road progress, the distance that Player 1 travels from their initial position (as a measure of efficiency), (iv) the minimum distance to Player 2, and (v) the time-averaged cost incurred over the full rollout, averaged over all of the trials. Using a simple reference policy, KLGame is able to coordinate at a much higher rate than other methods while maintaining safety.


\noindent \rebuttal{\textbf{Key Takeaway.} Mixing the game’s payoff with the expert’s reference policy allows the game solver to \textit{break out of a suboptimal equilibrium}.}


\begin{figure}[t!]
    \centering
    \includegraphics[width=0.95\columnwidth]{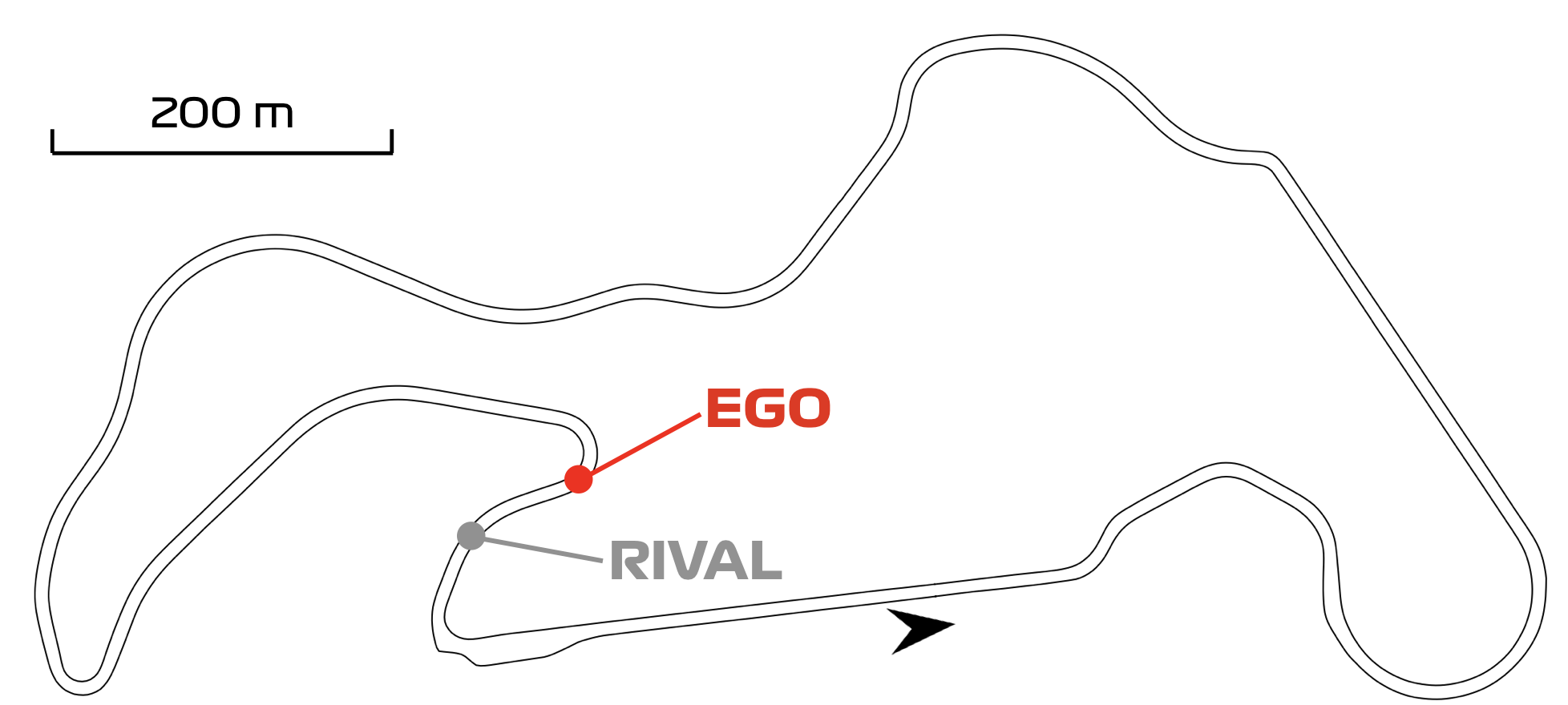}
    \caption{Initial positions of the ego and rival vehicles on the Thunderhill Raceway. The track direction is indicated by the black arrow.}
    \label{fig:thunderhill}
\end{figure}

\begin{figure}[t!]
    \centering
    \includegraphics[width=0.9\columnwidth]{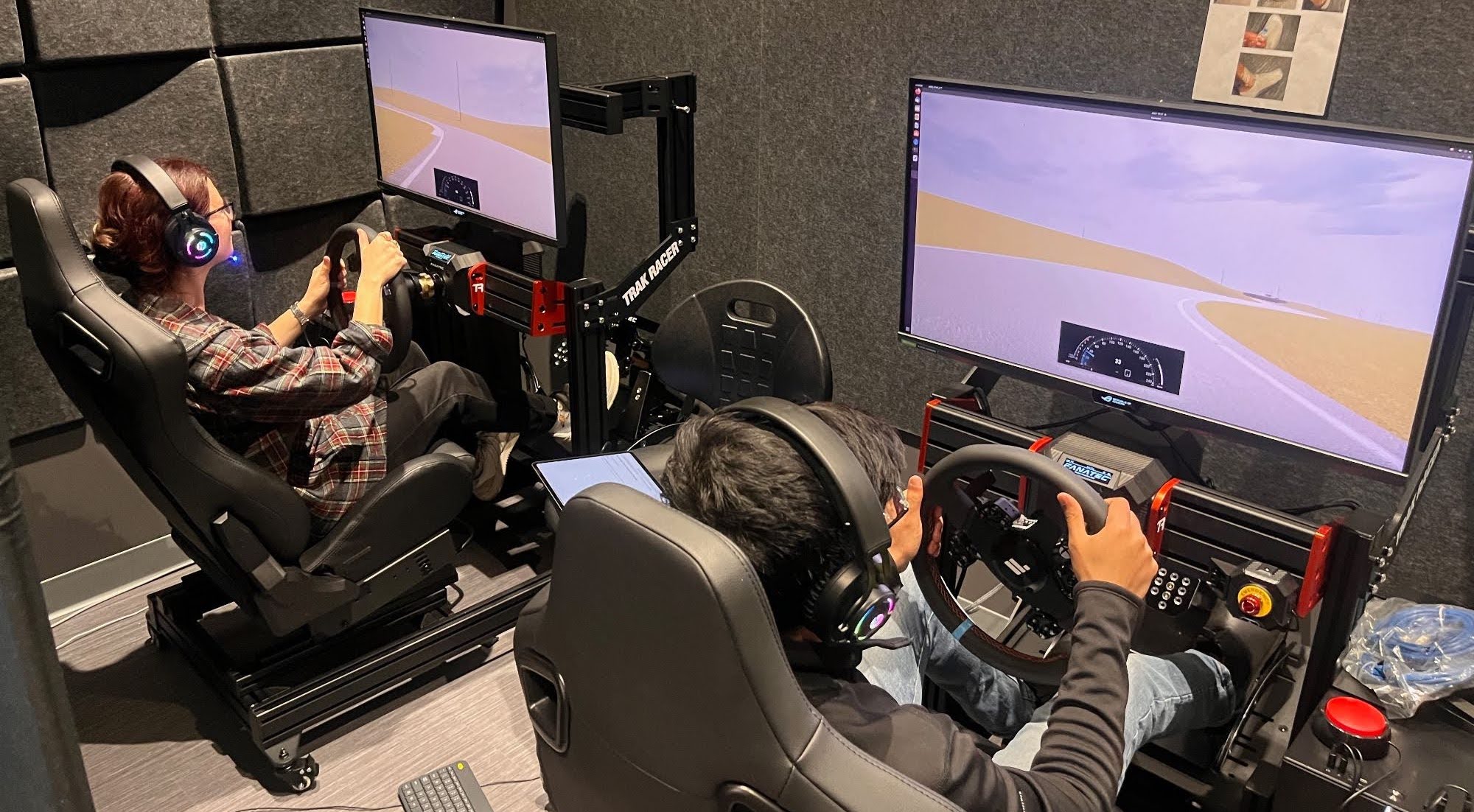}
    \caption{Driving simulator used for collecting expert racing data.}
    \label{fig:compactsim}
\end{figure}

\begin{figure*}[t!]
    \centering
    \includegraphics[width=2.0\columnwidth]{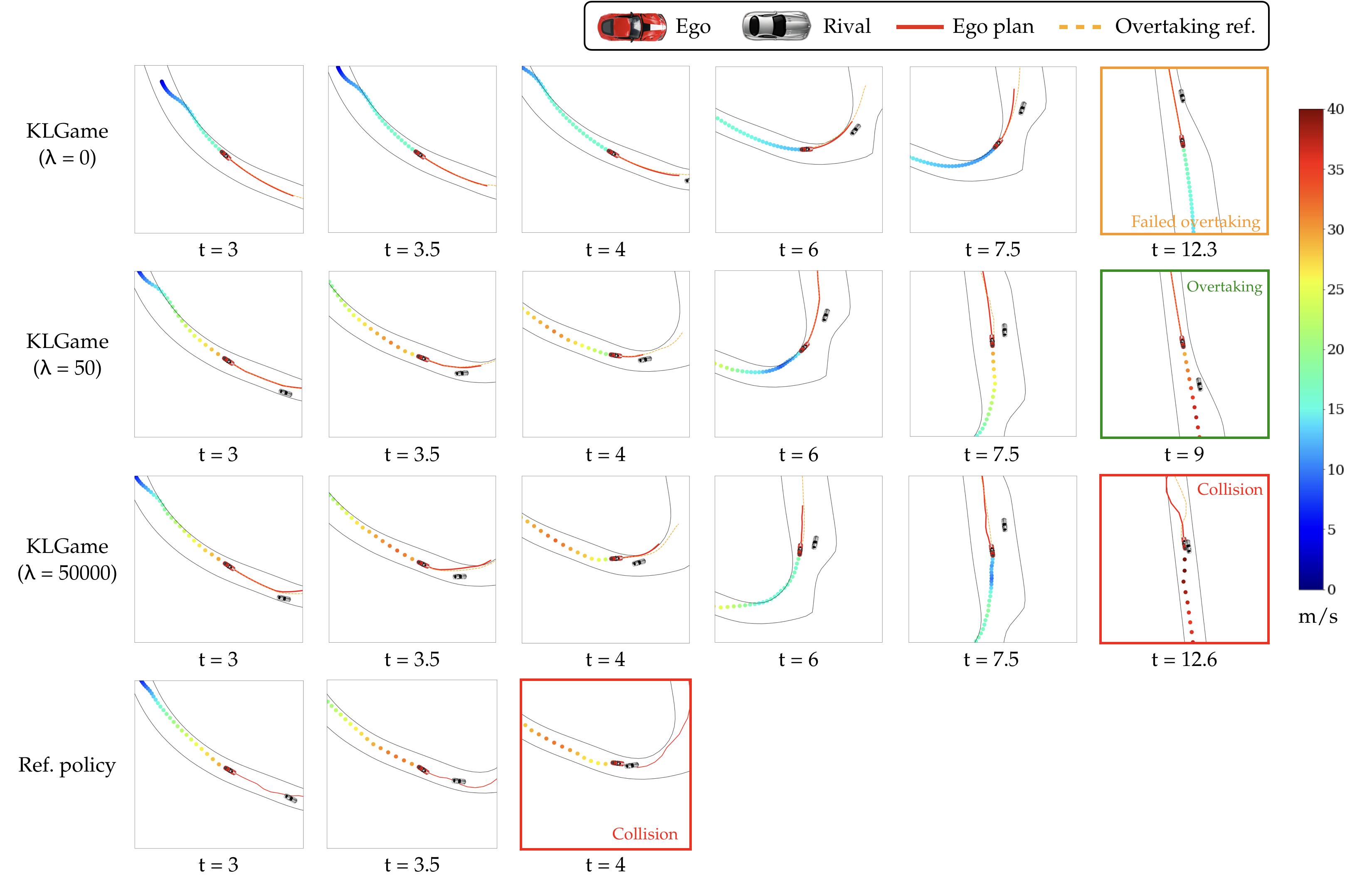}
    \caption{Autonomous racing with the ego car using the basic (unimodal) KLGame guided by a single overtaking reference policy learned from human demonstration.
    \emph{First row:} With un-regularized KL-Game ($\reg = 0$), the policy has no incentive to overtake and as a result, the ego cannot pass the rival.
    \emph{Second row:} For $\reg = 50$, a successful overtaking is observed.
    \emph{Third row:} For $\reg = 5 \times 10^4$, the ego vehicle's action, heavily influenced by the reference policy, leads to a collision after a failed overtake attempt.
    \emph{Fourth row:} The reference policy leads to a collision when the ego attempts to overtake the rival aggressively.
    }
    \vspace{-0.4cm}
    \label{fig:racing}
\end{figure*}

\subsection{Autonomous Car Racing}
\label{subsec:autonomous_car}

In this second experiment, we leverage KLGame within the more challenging context of autonomous car racing, delving into the advantages offered by the multi-modal KLGame policy framework.

\vspace{0.1cm}
\noindent \textbf{Setup.} We construct a simulated racing environment mirroring the Thunderhill Raceway in Willows, CA, USA, scaled accurately to 1:1 (see Fig.~\ref{fig:thunderhill}). In the race, the ego vehicle (red) is approaching a leading, rival vehicle (grey), and both vehicles are constrained to remain within the track boundaries. To make the scenario more challenging, we assume that only the ego vehicle has the responsibility to avoid a collision.

\begin{figure*}
    \centering
    \includegraphics[width=2.0\columnwidth]{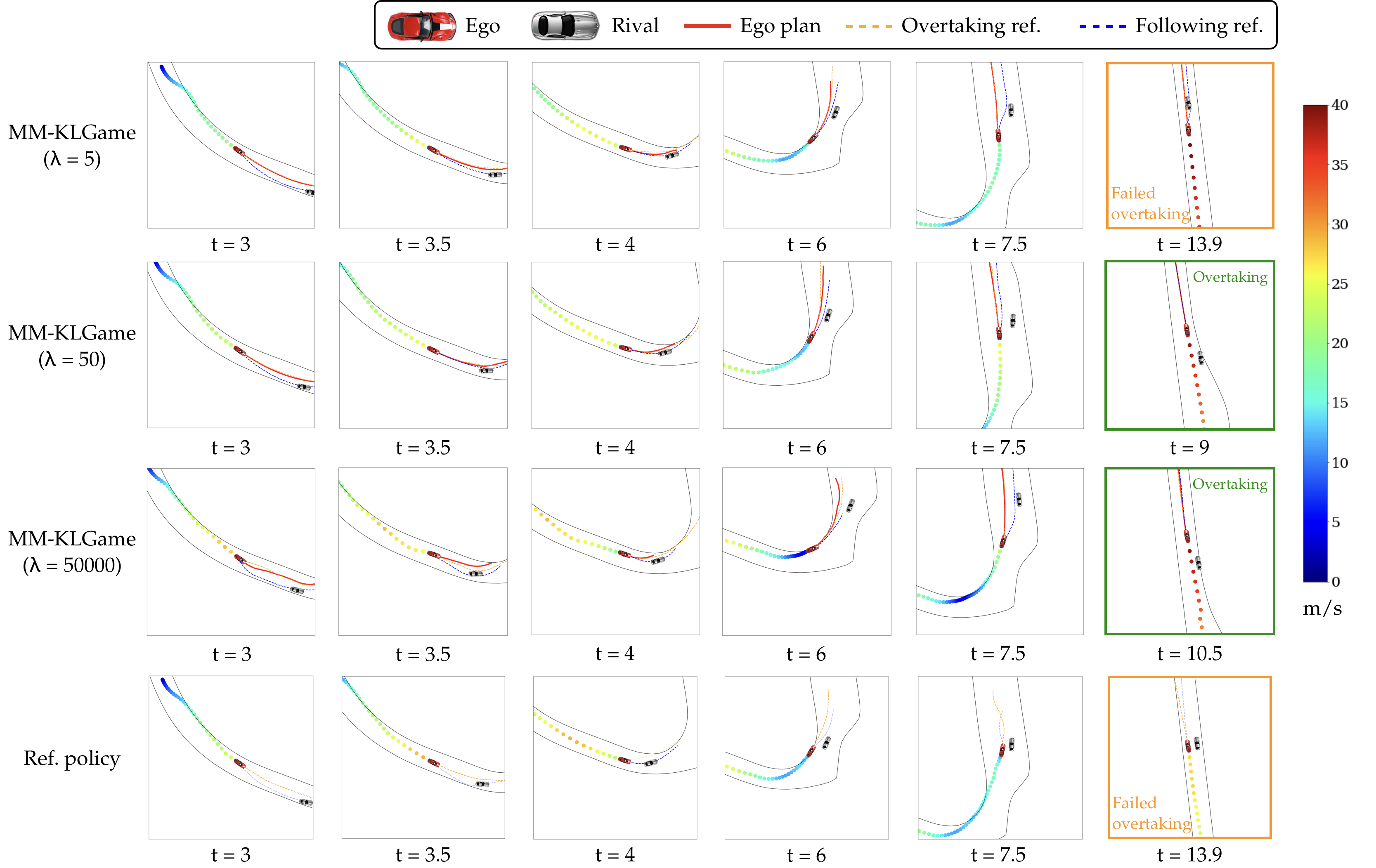}
    \caption{Autonomous racing with the ego car using a multi-modal KLGame policy mixing an overtaking and a following mode.
    \emph{First row:} With little regularization ($\reg = 5$), the ego is too conservative to overtake the rival.
    \emph{Second row:} For $\reg = 50$, the ego successfully overtakes the rival without incurring a collision.
    \emph{Third row:} For $\reg = 5 \times 10^4$, the MM-KLGame policy results in a safe overtake.
    \emph{Fourth row:} The multi-modal reference policy fails to overtake but averts collision. 
    }
    \vspace{-0.4cm}
    \label{fig:racing_mm}
\end{figure*}

\vspace{0.1cm}
\noindent \textbf{Reference policy and rival behaviors.} To obtain a reference policy, we use a driving simulator (Fig.~\ref{fig:compactsim}) to collect driving data (states, actions, track information, etc.) from races performed by two expert human racers in Carla~\cite{dosovitskiy2017carla}. 
With this data, we employ an inverse dynamic game training procedure inspired by works such as~\cite{li2023cost,liu2023learning}. This approach involves defining a set of basis functions for the game's cost---each representing a distinct racing specification like optimizing lap times, collision avoidance, and track boundary adherence---and learning a context-dependent cost weight model. The model, which takes as inputs driving data and outputs the cost weights, is learned as a deep neural network leveraging automatic differentiation in JAX~\cite{jax2018github}.
Given the limited volume of human-generated data, we use a dataset aggregation (DAgger) strategy~\cite{ross2011reduction} to enrich the training dataset for the reference policy.
To obtain a stochastic reference policy, we use a game-theoretic variant of the Boltzmann distribution~\cite[Section~3.2]{hu2023activeIJRR}, in which the Q-value function is obtained by solving the game with the inverse-learned cost weight model, which captures the noisily-rational aspect of human decision making.
Finally, we produce the rival's behavior with a defensive ILQGame policy~\cite{fridovich2020efficient}, whose parameters (cost terms and weights) are not accessible to either the reference or KLGame policies.

\noindent \textbf{Basic KLGame.} We first evaluate a basic unimodal KLGame policy, wherein the reference policy is simply the inverse-learned game policy that aggressively attempts to overtake the rival.
Snapshots from the simulation are plotted in Fig.~\ref{fig:racing}.
For $\lambda=0$, the ego simply follows the racing line (the time-optimal path of a race track) and fails to overtake the rival due to a lack of such incentive.
As $\lambda$ is increased to $50$, the ego car overtakes the rival safely.
When $\lambda = 5 \times 10^4$, the KLGame policy behaves closely to the aggressive reference policy, which leads to a collision.
Even though the ego's nominal game cost contains a collision avoidance term, it is ultimately overpowered by the regularization term weighted by the large value of $\lambda$ and is no longer capable of producing safe behaviors.
Finally, as a sanity check, we directly apply the reference policy to the ego vehicle and witness a collision, as expected.

\noindent \textbf{Multi-modal KLGame.} Motivated by the result shown above for \textit{Basic KLGame}, where blending only an aggressive reference policy can lead to unsafe outcomes, this experiment demonstrates how a \emph{multi-modal KLGame (MM-KLGame)} (c.f.~Sec.~\ref{sec:mmklg}) with an additional safety-enhancing mode performs in the racing scenario.
The second mode we use is a \emph{following} policy that controls the ego car to (i) follow the rival while keeping a constant distance when the rival is leading or (ii) defend against the rival after overtaking it.
During each planning cycle, instead of picking the mode via sampling, we use rollout-based safety filters~\cite[Section~3.3]{hsu2023sf} and perform a collision check for the KLGame-optimized trajectories associated with the overtaking mode: if they are collision-free, we choose the overtaking mode; otherwise, we pick the following mode.
Again, we vary the value of $\lambda$ and inspect the qualitative difference in the resulting ego behaviors.
The simulation snapshots are shown in Fig.~\ref{fig:racing_mm}.
When $\lambda$ takes a relatively small value, $\lambda = 5$, the MM-KLGame is not sufficiently guided by the overtaking policy and fails to overtake the rival.
When $\lambda=50$, the ego using MM-KLGame behaves similarly to the unimodal case and safely overtakes the rival.
However, in the case when $\lambda = 5 \times 10^4$, instead of incurring a collision as in the unimodal KLGame case, the MM-KLGame leads to a safe overtake thanks to the additional following mode in the multi-modal reference policy.
We also apply the reference policy to the ego using the same mode selection rule, \rebuttal{resulting in a \emph{shielded} inverse dynamic game policy.}
Interestingly, the ego car fails to overtake the rival due to a more frequent application of the following policy triggered by the unsafe overtaking policy, which ultimately slows down the ego vehicle.
This example also demonstrates the potential of KLGame using a multi-modal reference policy that mixes a data-driven policy and an optimization-based policy.

\setlength\tabcolsep{4pt}
\begin{table}[t!]
\rebuttal{
    \centering
    \caption{Racing Statistics}
    \label{tab: racing_stat}
    \begin{tabular}{l|ccccc}
        Method & Safe Rate $\uparrow$ & Overtaking Rate$\uparrow$   \\
        \hline \hline
        Inverse Dynamic Game~\cite{liu2023learning,li2023cost} & $0.73$ & $\mathbf{1.00}$ \\
        Shielded Inverse Dynamic Game & $0.90$ & $0.86$ \\
        \textit{KLGame (Ours)} & $0.85$ & $0.96$\\
        \textit{Multi-modal KLGame (Ours)} & $\mathbf{0.95}$ & $0.92$ \\
    \end{tabular}
}
\end{table}

\noindent \rebuttal{\textbf{Racing Statistics.}
Finally, in order to quantitatively assess the performance of our proposed KLGame and MM-KLGame policies, we compare them against two baselines: (1). inverse dynamic game with a neural cost model (the reference policy), and (2). the shielded inverse dynamic game policy.
We simulate 100 \textit{randomized} racing scenarios, each with different initial positions on the Thunderhill track (for both the ego and rival vehicles), as well as varying levels of aggressiveness in the rival's defending behaviors. Each scenario is evaluated for each policy using the same random seed.
We use the same blending factor $\lambda$ for both KLGame and MM-KLGame policies across all simulation instances.
The safe rate and overtaking rate statistics obtained from these trials are shown in Table~\ref{tab: racing_stat}.
Since the reference policy is learned from a dataset containing aggressive overtaking demonstrations, it achieves the highest overtaking rate but falls short of the safe rate.
The (uni-modal) KLGame policy significantly improves the safe rate by $12\%$ compared to the reference policy, albeit at the cost of a $4\%$ reduction in the overtaking rate.
The multi-modal KLGame policy, due to the additional car following mode, achieves the highest safety rate ($95\%$) among the four policies, while maintaining a decent overtaking rate ($92\%$).
Finally, the shielded inverse game policy, despite reaching a high safe rate, is ultimately the most conservative policy due to the planner-agnostic safety filter.
}

\noindent \rebuttal{\textbf{Key Takeaway.} Our blended multi-modal KLGame policy can balance between competitiveness and cautiousness in \emph{fast and rivalrous} interactions such as car racing.
}


\subsection{Scene-Consistent Simulated Agents in Waymax}
\label{subsec:waymax}

Modern trajectory forecasting models \cite{Nayakanti2022Wayformer, shi2022motion, shi2023mtr++} learn natural driving behaviors from data and generate accurate forecasts across different scenarios without the need for manual feature selection. However, these models may encounter difficulties when they are directly used to simulate interactive multi-agent scenarios, as the predicted marginal trajectories may not be compatible with others \cite{chen2022scept}. In our experiment, we demonstrate the ability of the proposed KLGame to capitalize on data-driven priors for generating natural and scene-consistent reference policies for interactive and safety-critical driving scenarios involving multiple agents.

\noindent\textbf{MTR Reference Policy.} In this experiment, we utilize a pre-trained MTR \cite{shi2022motion} motion prediction model, which outputs trajectories represented by a Gaussian mixture model (GMM) as the reference policy for KLGame. We then use the MTR-guided reference policy to emulate agent behaviors on Waymax \cite{gulino2023waymax} with the interactive validation dataset of Waymo Open Motion Dataset \cite{ettinger2021waymo}. Because the MTR model only outputs a GMM of future trajectories for each agent, we estimate the control sequences leading to the mean trajectory of the corresponding modes via optimization. The reference policy can be constructed from a mixture of control sequences with a fixed covariance and mode probabilities from MTR. This step can be skipped if the model can directly generate a policy given the state. 

\noindent\textbf{KLGame on Waymax.} For our implementation of KLGame on Waymax, we use \textit{only two} generic optimization targets as cost functions: collision distance and control magnitude regularization, with the rest of the agent's optimization objective coming from the MTR reference policy guidance. Alongside the control regularization, each agent is also tasked with maximizing its signed distance to the closest players and off-road areas. These objectives prompt the agents to prevent collisions effectively and adhere to the road. The signed distance between agents is calculated as the Minkowski difference between two footprints, while the distance to the off-road areas is obtained from road boundaries. In our experiments, all agents are replanned every second in a receding horizon fashion with a 2-second planning horizon and 0.1s discrete timestep.  


\begin{table*}[]
\rebuttal{
    \centering
    \begin{tabular}{l|cccccc}
        Method &  Collision Rate $\downarrow$ & ADE [3s] (m) $\downarrow$ & ADE [5s] (m) $\downarrow$ & ADE [8s] (m) $\downarrow$ & ADE [avg] (m) \\
        \hline \hline
        ILQGames \cite{fridovich2020efficient} & $0.13$ &  $0.23$ & $1.02$ & $3.34$ & $1.53$ \\
        MaxEnt \cite{mehr2023maximum} & $0.13$ & $0.23$ & $1.02$ & $3.34$ & $1.53$\\
        IDM \cite{treiber2000congested} & $0.10$ &  $0.81$ & $3.03$ & $8.04$ & $3.96$\\
        MTR \cite{shi2022motion} & $0.12$  & $0.15$ & $0.70$ & $2.19$ & $1.01$\\
        \textit{KLGame (Ours)} & $0.09$ &  $0.18$ & $0.80$ & $2.45$ & $1.14$\\
    \end{tabular}
    \caption{Results for a KLGame ego-agent planner against several other ego-agent baselines on Waymax \cite{gulino2023waymax} when the other agent follows the IDM~\cite{treiber2000congested} model.}
    \label{tab: waymax_experiment_results_idm}
}
\end{table*}

\begin{figure}[t!]
    \centering
    \includegraphics[width=0.5\textwidth]{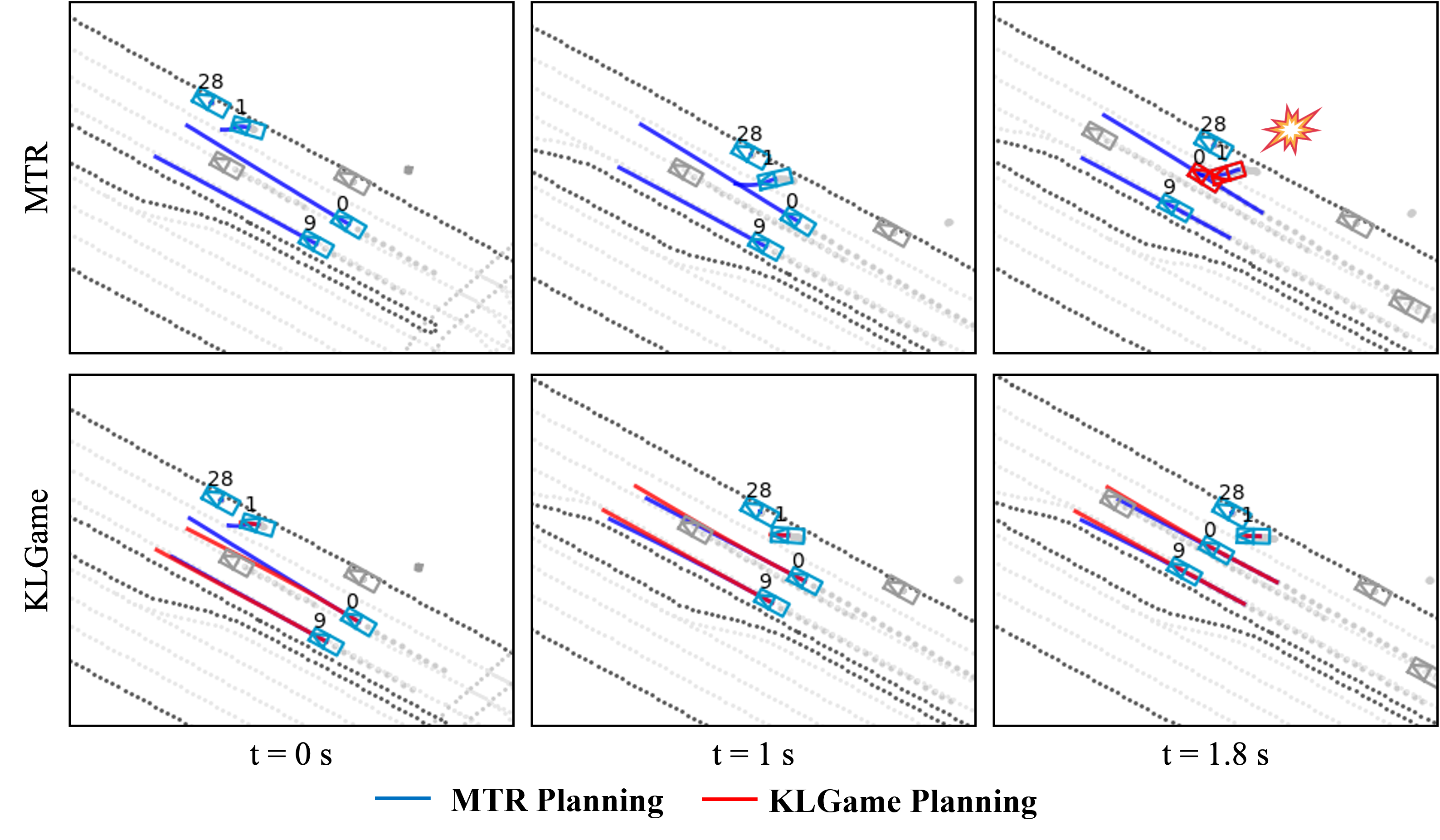}
    \caption{As car 1 pulls out from the parked vehicle, KLGame optimizes the joint trajectories, guiding both car 0 and car 1 to exhibit yielding behaviors to prevent a collision. However, the reference MTR policy, in contrast, results in a collision.}
    \label{fig: waymax_scenario_1}
\end{figure}

The first scenario (Fig. \ref{fig: waymax_scenario_1}) consists of four agents in the game, where one of the agents (vehicle 1) merges into the oncoming traffic from behind a parked vehicle (vehicle 28). In doing so, vehicle 1 obstructs vehicle 0's path. The socially optimal trajectories given by KLGame at $t=0$ are such that vehicle 0's modifies its trajectory to slightly veer left and decelerates agent 1. At the replanning stage $t=1$, despite MTR’s reference policy leading vehicle 9 towards the edge of the road at the end of the planning horizon, KLGame manages to optimize the trajectory, ensuring it adheres to the lane center, without requiring explicit information about the lane and reference path. The simulated scenario using the proposed KLGame maintains the scene collision-free throughout the episode, while marginal prediction from MTR leads to a collision between vehicles 0 and 1 at $t=1.8s$. 

In the second scenario (Fig. \ref{fig: waymax_scenario_3}), vehicle 0 merges onto the main road, interacting with vehicle 1. The standard MTR policy attempts to slow down vehicle 0, but results in a collision. In contrast, KLGame refines the joint trajectories, planning for vehicle 0 to change lanes, thereby avoiding a collision. While the aggressive evasive maneuver for vehicle 1 at $t=3$ initially risks taking the vehicle off the road, the reference policy regularization combined  with the optimization goal assist  vehicle 1 in maintaining the proper course on subsequent replanning  stages.

Finally, Fig. \ref{fig: waymax_scenario_2} shows that KLGame can yield diverse, yet scene-consistent, traffic plans, particularly when paired with a multi-modal reference policy. In this intersection scenario, agent 44 approaches the intersection from the wrong side of the road. This unlikely event leads to the marginal predictions from MTR having a hard time keeping the scenario collision-free. In the first rollout, as vehicle 44 opts for a mode to proceed straight, KLGame plans emergency maneuvers for both vehicles to avoid collision. In the subsequent rollout, vehicle 44 samples the mode of turning right. Despite a proximate interaction, KLGame computes the optimal joint trajectory, thereby successfully preventing a collision. 

\vspace{0.3cm}
\noindent
\rebuttal{\textbf{Waymax Statistics.} We compare our proposed KLGame method against four baselines: (1) ILGames \cite{fridovich2020efficient}, (2) MaxEnt dynamic game \cite{mehr2023maximum}, (3) the Intelligent Driver Model (IDM) \cite{treiber2000congested}, and (4) MTR \cite{shi2022motion}.  To measure the robustness of each planner to modeling errors, we evaluated each method against an IDM agent in $500$ scenarios from the Waymo Open Motion Dataset \cite{ettinger2021large} \textit{interactive validation } split, which is composed of especially challenging interactive scenarios selected by \cite{zhang2023cat}. The safe rate, collision rate, and ADE along each trajectory are reported in Table \ref{tab: waymax_experiment_results_idm}. Since the ILQGames and MaxEnt dynamic game baselines do not use a multi-modal reference policy, we instead incorporate reference tracking via the most likely MTR mode. The multimodal KLGame significantly improves on the unimodal ILQGames and MaxEnt dynamic game. Furthermore, KLGame achieves the lowest collision rate among all methods, including data-driven MTR \cite{shi2022motion} while maintaining the second-lowest ADE scores by a large margin. With only a $12\%$ increase in ADE ($2.45~m$ with KLGame versus $2.19~m$ with MTR), KLGame reduces the number of collisions by $25\%$ (45 with KLGame versus 60 with MTR). Overall, these results suggest that the multi-modal policy solved by KLGame encourages safety even when the other agents in the game are not modeled perfectly. }



\begin{figure*}[h!]
    \centering
    \includegraphics[width=\textwidth]{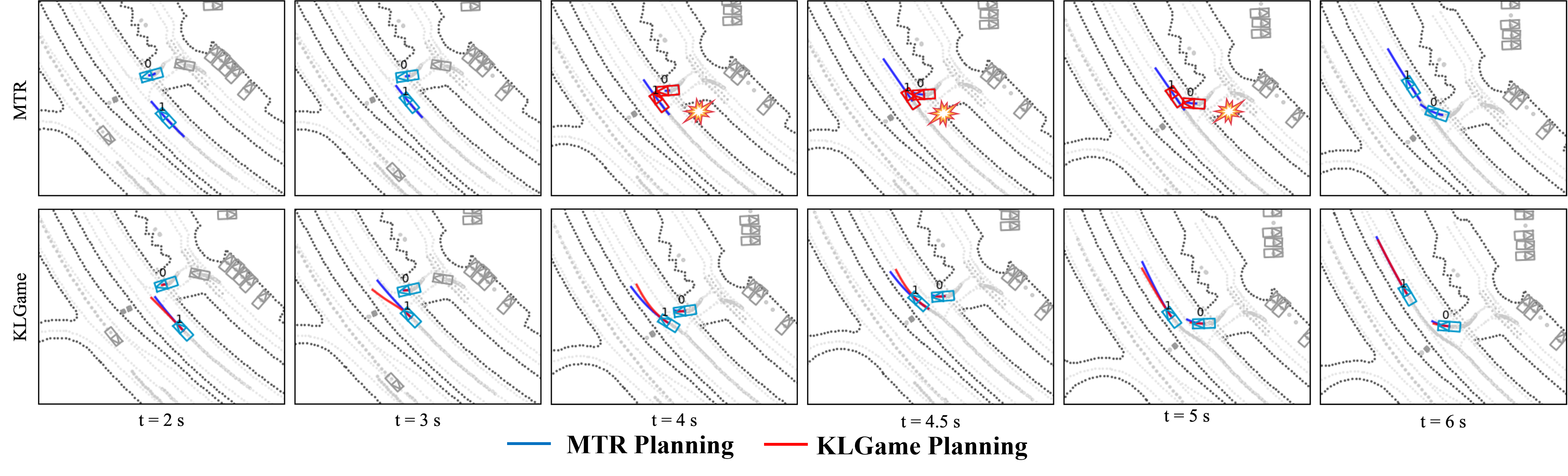}
    \caption{  
    As car 1 merges into the main road, the MTR-derived policy slows down car 0 but ultimately results in a collision. In comparison, KLGame optimizes the joint trajectories by planning for car 0 to switch lanes and yield, thus avoiding a collision.
    }
    \label{fig: waymax_scenario_3}
\end{figure*}

\begin{figure*}[h!]
    \centering
    \includegraphics[width=\textwidth]{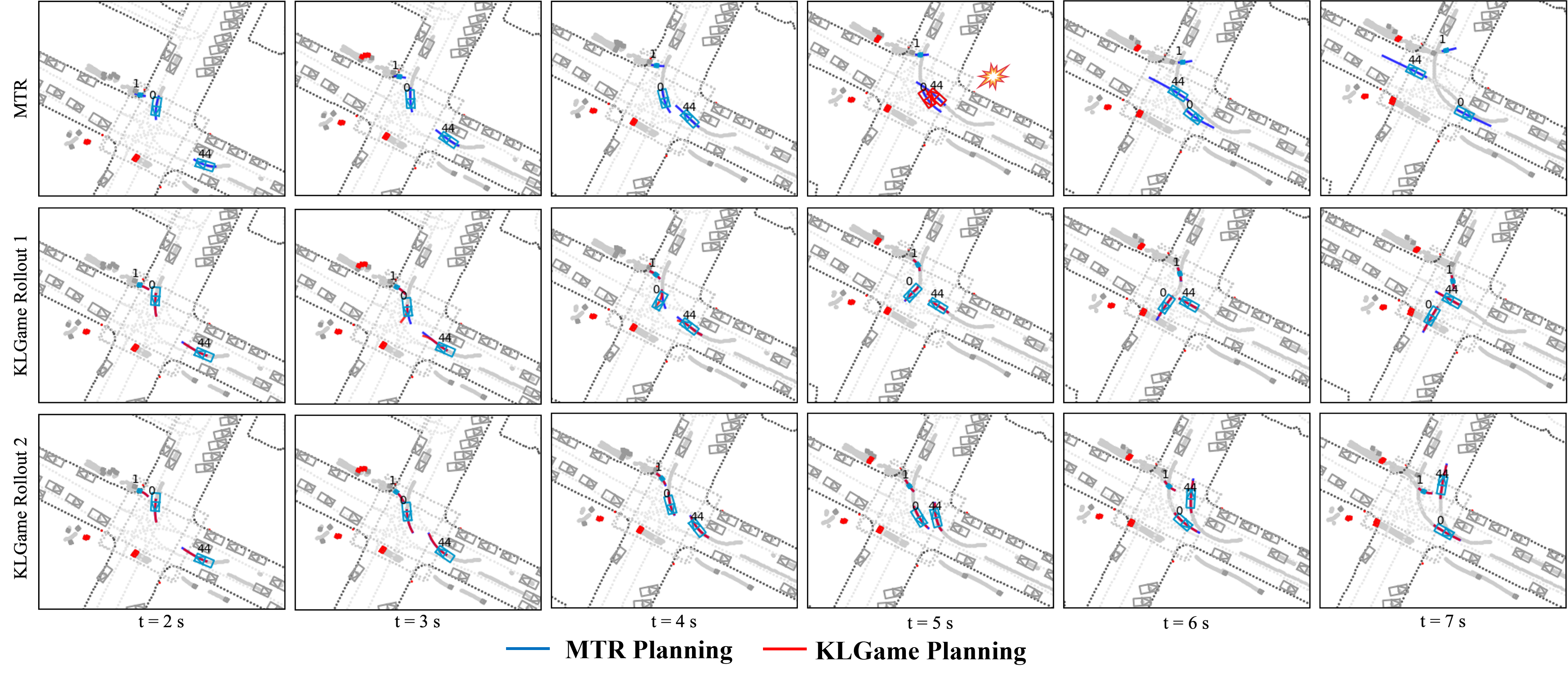}
    \caption{Car 44 approaches the intersection from the incorrect side of the road. In the first rollout, car 44 opts for a mode to proceed straight, causing KLGame to plan an emergency maneuver to circumvent collision. In the subsequent rollout, car 44 samples the mode of turning right. Despite a proximate interaction, KLGame computes the optimal joint trajectory, thereby successfully averting a collision. This demonstrates KLGame's ability to generate diverse trajectories with a multi-modal reference policy.
    }
    \label{fig: waymax_scenario_2}
\end{figure*}

\vspace{0.3cm}
\noindent \rebuttal{\textbf{Key Takeaway.}
%
The KLGame planner allows integrating the state-of-the-art large-scale, data-driven behavior model with an optimization-based game policy to improve safety for complex tasks such as urban autonomous driving \emph{at scale}.
}

\rebuttal{
\begin{figure}[h]
    \centering
    \includegraphics[width=0.4\textwidth]{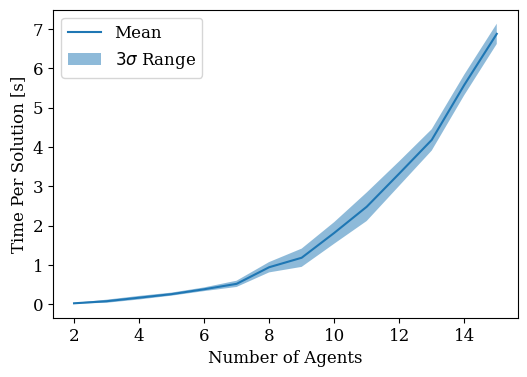}
    \caption{We examined the worst case runtime and its 3-$\sigma$ limits for KLGame with different number agents. In each trial, We repeated each trails for 80 times and ensure that the KLGame solver iterates 15 times with 15 line searches.}
    \label{fig:runtime}
\end{figure}
}
\section{Limitations and Future Work}
In this paper, we empirically demonstrate that, by tuning the KL-regularization weight $\reg$, we can trade off task performance with data-driven (human) behaviors. However, we have not yet delved into what the ``optimal'' value of $\reg$ should be and how to find it.
Such questions become more relevant in human-AI shared autonomy, which can naturally benefit from applying KLGame to ease the issue of \emph{automation surprise}~\cite{sarter1997team,jamieson2022b737} via blending the AI policy with the data-driven human policy.
We believe KLGame can be extended to account for human expectations by identifying the human-preferred value of $\reg$ via, e.g., differentiable game-theoretic planning~\cite{peters2022rss,li2023cost,liu2023learning} or inference-based planning techniques~\cite{fisacBHFWTD18, peters2020inference, tian2022safety, hu2023activeIJRR}.
In addition, motivated by the growing need for \emph{interactive} driving datasets and benchmarks, we see an open opportunity to use KLGame for generating scene-consistent and interaction-rich traffic predictions for autonomous driving.


\section{Conclusion}


In this work, we propose KLGame, an interactive planning and prediction framework that blends data-driven priors with task-optimal behaviors in a game-theoretic setting. By incorporating information about the uncertainty of other agents' intents---represented by a multi-modal probability distribution---KLGame enables a robot to plan for multiple future scenarios while also providing a mechanism for plan generation through random sampling. Through detailed simulations and real-world autonomous driving data, we demonstrate KLGame's ability to
incorporate both optimization-based and data-driven priors
into robot
motion planning.
In essence, KLGame represents a significant stride in integrating data-driven priors with game-theoretic planning, showcasing its prowess in the realm of autonomous driving.
Beyond the scope of self-driving cars, the general philosophy of KLGame opens doors to broader applications in human-centered robotics and artificial intelligence. 

\bibliographystyle{IEEEtran}
\balance
\bibliography{IEEEabrv,references.bib}

\begin{thebibliography}{100}
\providecommand{\url}[1]{#1}
\csname url@rmstyle\endcsname
\providecommand{\newblock}{\relax}
\providecommand{\bibinfo}[2]{#2}
\providecommand\BIBentrySTDinterwordspacing{\spaceskip=0pt\relax}
\providecommand\BIBentryALTinterwordstretchfactor{4}
\providecommand\BIBentryALTinterwordspacing{\spaceskip=\fontdimen2\font plus
\BIBentryALTinterwordstretchfactor\fontdimen3\font minus
  \fontdimen4\font\relax}
\providecommand\BIBforeignlanguage[2]{{%
\expandafter\ifx\csname l@#1\endcsname\relax
\typeout{** WARNING: IEEEtran.bst: No hyphenation pattern has been}%
\typeout{** loaded for the language `#1'. Using the pattern for}%
\typeout{** the default language instead.}%
\else
\language=\csname l@#1\endcsname
\fi
#2}}

\bibitem{bacsar1998dynamic}
T.~Ba{\c{s}}ar and G.~J. Olsder, \emph{Dynamic noncooperative game
  theory}.\hskip 1em plus 0.5em minus 0.4em\relax SIAM, 1998.

\bibitem{espinoza2022deep}
J.~L.~V. Espinoza, A.~Liniger, W.~Schwarting, D.~Rus, and L.~Van~Gool, ``Deep
  interactive motion prediction and planning: Playing games with motion
  prediction models,'' in \emph{L4DC}.\hskip 1em plus 0.5em minus 0.4em\relax
  PMLR, 2022, pp. 1006--1019.

\bibitem{lindemann2023safe}
L.~Lindemann, M.~Cleaveland, G.~Shim, and G.~J. Pappas, ``Safe planning in
  dynamic environments using conformal prediction,'' \emph{IEEE Robotics and
  Automation Letters}, 2023.

\bibitem{mehr2023maximum}
N.~Mehr, M.~Wang, M.~Bhatt, and M.~Schwager, ``Maximum-entropy multi-agent
  dynamic games: Forward and inverse solutions,'' \emph{IEEE Transactions on
  Robotics}, 2023.

\bibitem{fridovich2020efficient}
D.~Fridovich-Keil, E.~Ratner, L.~Peters, A.~D. Dragan, and C.~J. Tomlin,
  ``Efficient iterative linear-quadratic approximations for nonlinear
  multi-player general-sum differential games,'' in \emph{2020 IEEE
  international conference on robotics and automation (ICRA)}.\hskip 1em plus
  0.5em minus 0.4em\relax IEEE, 2020, pp. 1475--1481.

\bibitem{shi2023mtr++}
S.~Shi, L.~Jiang, D.~Dai, and B.~Schiele, ``{MTR++}: Multi-agent motion
  prediction with symmetric scene modeling and guided intention querying,''
  \emph{arXiv preprint arXiv:2306.17770}, 2023.

\bibitem{schwarting2019social}
W.~Schwarting, A.~Pierson, J.~Alonso-Mora, S.~Karaman, and D.~Rus, ``Social
  behavior for autonomous vehicles,'' \emph{PNAS}, vol. 116, no.~50, pp.
  24\,972--24\,978, 2019.

\bibitem{wang2021game}
M.~Wang, Z.~Wang, J.~Talbot, J.~C. Gerdes, and M.~Schwager, ``Game-theoretic
  planning for self-driving cars in multivehicle competitive scenarios,''
  \emph{IEEE T-RO}, vol.~37, no.~4, pp. 1313--1325, 2021.

\bibitem{hu2023activeIJRR}
H.~Hu, D.~Isele, S.~Bae, and J.~F. Fisac, ``Active uncertainty reduction for
  safe and efficient interaction planning: A shielding-aware dual control
  approach,'' \emph{The International Journal of Robotics Research}, 2023.

\bibitem{sun2021move}
M.~Sun, F.~Baldini, P.~Trautman, and T.~Murphey, ``Move beyond trajectories:
  Distribution space coupling for crowd navigation,'' ser. Robotics: Science
  and Systems, 2021.

\bibitem{williams2023distributed}
Z.~Williams, J.~Chen, and N.~Mehr, ``Distributed potential ilqr: Scalable
  game-theoretic trajectory planning for multi-agent interactions,''
  \emph{arXiv preprint arXiv:2303.04842}, 2023.

\bibitem{hu2020non}
H.~Hu, K.~Gatsis, M.~Morari, and G.~J. Pappas, ``{Non-cooperative distributed
  MPC with iterative learning},'' \emph{IFAC-PapersOnLine}, vol.~53, no.~2, pp.
  5225--5232, 2020.

\bibitem{spica2020real}
R.~Spica, E.~Cristofalo, Z.~Wang, E.~Montijano, and M.~Schwager, ``A real-time
  game theoretic planner for autonomous two-player drone racing,'' \emph{IEEE
  Transactions on Robotics}, vol.~36, no.~5, pp. 1389--1403, 2020.

\bibitem{music2020haptic}
S.~Musi{\'c} and S.~Hirche, ``Haptic shared control for human-robot
  collaboration: a game-theoretical approach,'' \emph{IFAC-PapersOnLine},
  vol.~53, no.~2, pp. 10\,216--10\,222, 2020.

\bibitem{geiger2021learning}
P.~Geiger and C.-N. Straehle, ``Learning game-theoretic models of multiagent
  trajectories using implicit layers,'' in \emph{{AAAI}}, vol.~35, no.~6, 2021,
  pp. 4950--4958.

\bibitem{cleac2020algames}
S.~Le~Cleac'h, M.~Schwager, Z.~Manchester, \emph{et~al.}, ``{ALGAMES}: A fast
  solver for constrained dynamic games,'' in \emph{RSS}, 2020.

\bibitem{liu2023learning}
X.~Liu, L.~Peters, and J.~Alonso-Mora, ``Learning to play trajectory games
  against opponents with unknown objectives,'' \emph{IEEE Robotics and
  Automation Letters}, 2023.

\bibitem{le2021lucidgames}
S.~Le~Cleac’h, M.~Schwager, and Z.~Manchester, ``{LUCIDgames}: Online
  unscented inverse dynamic games for adaptive trajectory prediction and
  planning,'' \emph{IEEE RA-L}, vol.~6, no.~3, pp. 5485--5492, 2021.

\bibitem{Fisac2019-hierarchical}
J.~F. Fisac, E.~Bronstein, E.~Stefansson, D.~Sadigh, S.~S. Sastry, and A.~D.
  Dragan, ``Hierarchical {Game-Theoretic} planning for autonomous vehicles,''
  in \emph{ICRA}, May 2019, pp. 9590--9596.

\bibitem{hu2022activeWAFR}
H.~Hu and J.~F. Fisac, ``Active uncertainty reduction for human-robot
  interaction: An implicit dual control approach,'' in \emph{Algorithmic
  Foundations of Robotics XV}, 2022, pp. 385--401.

\bibitem{schwarting2021stochastic}
W.~Schwarting, A.~Pierson, S.~Karaman, and D.~Rus, ``Stochastic dynamic games
  in belief space,'' \emph{IEEE T-RO}, vol.~37, no.~6, pp. 2157--2172, 2021.

\bibitem{steinberger2017road}
F.~Steinberger, R.~Schroeter, and C.~N. Watling, ``From road distraction to
  safe driving: Evaluating the effects of boredom and gamification on driving
  behaviour, physiological arousal, and subjective experience,''
  \emph{Computers in Human Behavior}, vol.~75, pp. 714--726, 2017.

\bibitem{talebpour2015modeling}
A.~Talebpour, H.~S. Mahmassani, and S.~H. Hamdar, ``Modeling lane-changing
  behavior in a connected environment: A game theory approach,''
  \emph{Transportation Research Procedia}, vol.~7, pp. 420--440, 2015.

\bibitem{hu2023belgame}
H.~Hu, Z.~Zhang, K.~Nakamura, A.~Bajcsy, and J.~F. Fisac, ``Deception game:
  Closing the safety-learning loop in interactive robot autonomy,'' in
  \emph{Conference on Robot Learning}.\hskip 1em plus 0.5em minus 0.4em\relax
  PMLR, 2023, pp. 3830--3850.

\bibitem{young2017toward}
K.~L. Young, S.~Koppel, and J.~L. Charlton, ``Toward best practice in human
  machine interface design for older drivers: A review of current design
  guidelines,'' \emph{Accident Analysis \& Prevention}, vol. 106, pp. 460--467,
  2017.

\bibitem{ziebart2008maximum}
B.~D. Ziebart, A.~L. Maas, J.~A. Bagnell, A.~K. Dey, \emph{et~al.}, ``Maximum
  entropy inverse reinforcement learning.'' in \emph{{AAAI}}, vol.~8, 2008, pp.
  1433--1438.

\bibitem{wulfmeier2017large}
M.~Wulfmeier, D.~Rao, D.~Z. Wang, P.~Ondruska, and I.~Posner, ``Large-scale
  cost function learning for path planning using deep inverse reinforcement
  learning,'' \emph{The International Journal of Robotics Research}, vol.~36,
  no.~10, pp. 1073--1087, 2017.

\bibitem{evens2021}
B.~Evens, M.~Schuurmans, and P.~Patrinos, ``Learning mpc for interaction-aware
  autonomous driving: A game-theoretic approach,'' in \emph{2022 European
  Control Conference (ECC)}, 2022, pp. 34--39.

\bibitem{phan2022driving}
T.~Phan-Minh, F.~Howington, T.-S. Chu, S.~U. Lee, M.~S. Tomov, N.~Li, C.~Dicle,
  S.~Findler, F.~Suarez-Ruiz, R.~Beaudoin, \emph{et~al.}, ``Driving in real
  life with inverse reinforcement learning,'' \emph{arXiv preprint
  arXiv:2206.03004}, 2022.

\bibitem{peters2020inference}
L.~Peters, D.~Fridovich-Keil, C.~J. Tomlin, and Z.~N. Sunberg,
  ``Inference-based strategy alignment for general-sum differential games,''
  \emph{arXiv preprint arXiv:2002.04354}, 2020.

\bibitem{so2022maximum}
O.~So, Z.~Wang, and E.~A. Theodorou, ``Maximum entropy differential dynamic
  programming,'' in \emph{2022 International Conference on Robotics and
  Automation (ICRA)}.\hskip 1em plus 0.5em minus 0.4em\relax IEEE, 2022, pp.
  3422--3428.

\bibitem{so2023mpogames}
O.~So, P.~Drews, T.~Balch, V.~Dimitrov, G.~Rosman, and E.~A. Theodorou,
  ``{MPOGames}: Efficient multimodal partially observable dynamic games,'' in
  \emph{ICRA}.\hskip 1em plus 0.5em minus 0.4em\relax IEEE, 2023, pp.
  3189--3196.

\bibitem{littman1995learning}
M.~L. Littman, A.~R. Cassandra, and L.~P. Kaelbling, ``Learning policies for
  partially observable environments: Scaling up,'' in \emph{Machine Learning
  Proceedings 1995}.\hskip 1em plus 0.5em minus 0.4em\relax Elsevier, 1995, pp.
  362--370.

\bibitem{peters2023contingency}
L.~Peters, A.~Bajcsy, C.-Y. Chiu, D.~Fridovich-Keil, F.~Laine, L.~Ferranti, and
  J.~Alonso-Mora, ``Contingency games for multi-agent interaction,''
  \emph{arXiv preprint arXiv:2304.05483}, 2023.

\bibitem{chen2021interactive}
Y.~Chen, U.~Rosolia, W.~Ubellacker, N.~Csomay-Shanklin, and A.~D. Ames,
  ``Interactive multi-modal motion planning with branch model predictive
  control,'' \emph{IEEE Robotics and Automation Letters}, vol.~7, no.~2, pp.
  5365--5372, 2022.

\bibitem{hu2023emergent}
H.~Hu, K.~Nakamura, K.-C. Hsu, N.~E. Leonard, and J.~F. Fisac, ``Emergent
  coordination through game-induced nonlinear opinion dynamics,'' in \emph{2023
  62nd IEEE Conference on Decision and Control (CDC)}.\hskip 1em plus 0.5em
  minus 0.4em\relax IEEE, 2023, pp. 8122--8129.

\bibitem{peters2022rss}
L.~Peters, D.~Fridovich-Keil, L.~Ferranti, C.~Stachniss, J.~Alonso-Mora, and
  F.~Laine, ``Learning mixed strategies in trajectory games,'' in
  \emph{Proc.~of Robotics: Science and Systems (RSS)}, 2022.

\bibitem{li2023cost}
J.~Li, C.-Y. Chiu, L.~Peters, S.~Sojoudi, C.~Tomlin, and D.~Fridovich-Keil,
  ``Cost inference for feedback dynamic games from noisy partial state
  observations and incomplete trajectories,'' \emph{arXiv preprint
  arXiv:2301.01398}, 2023.

\bibitem{diehl2023energy}
C.~Diehl, T.~Klosek, M.~Krueger, N.~Murzyn, T.~Osterburg, and T.~Bertram,
  ``Energy-based potential games for joint motion forecasting and control,'' in
  \emph{7th Annual Conference on Robot Learning}, 2023.

\bibitem{todorov2004optimality}
E.~Todorov, ``Optimality principles in sensorimotor control,'' \emph{Nature
  neuroscience}, vol.~7, no.~9, pp. 907--915, 2004.

\bibitem{bartumeus2008fractal}
F.~Bartumeus and S.~A. Levin, ``Fractal reorientation clocks: Linking animal
  behavior to statistical patterns of search,'' \emph{Proceedings of the
  National Academy of Sciences}, vol. 105, no.~49, pp. 19\,072--19\,077, 2008.

\bibitem{ettinger2021waymo}
S.~Ettinger, S.~Cheng, B.~Caine, C.~Liu, H.~Zhao, S.~Pradhan, Y.~Chai, B.~Sapp,
  C.~R. Qi, Y.~Zhou, \emph{et~al.}, ``Large scale interactive motion
  forecasting for autonomous driving: The waymo open motion dataset,'' in
  \emph{CVPR}, 2021, pp. 9710--9719.

\bibitem{kim2020hamilton}
J.~Kim and I.~Yang, ``Hamilton-jacobi-bellman equations for maximum entropy
  optimal control,'' \emph{arXiv preprint arXiv:2009.13097}, 2020.

\bibitem{theodorou2012relative}
E.~A. Theodorou and E.~Todorov, ``Relative entropy and free energy dualities:
  Connections to path integral and kl control,'' in \emph{2012 ieee 51st ieee
  conference on decision and control (cdc)}.\hskip 1em plus 0.5em minus
  0.4em\relax IEEE, 2012, pp. 1466--1473.

\bibitem{haarnoja2017reinforcement}
T.~Haarnoja, H.~Tang, P.~Abbeel, and S.~Levine, ``Reinforcement learning with
  deep energy-based policies,'' in \emph{International conference on machine
  learning}.\hskip 1em plus 0.5em minus 0.4em\relax PMLR, 2017, pp. 1352--1361.

\bibitem{garg2021iq}
D.~Garg, S.~Chakraborty, C.~Cundy, J.~Song, and S.~Ermon, ``Iq-learn: Inverse
  soft-q learning for imitation,'' \emph{Advances in Neural Information
  Processing Systems}, vol.~34, pp. 4028--4039, 2021.

\bibitem{todorov2006linearly}
E.~Todorov, ``Linearly-solvable markov decision problems,'' \emph{Advances in
  neural information processing systems}, vol.~19, 2006.

\bibitem{todorov2009compositionality}
------, ``Compositionality of optimal control laws,'' \emph{Advances in neural
  information processing systems}, vol.~22, 2009.

\bibitem{guan2014online}
P.~Guan, M.~Raginsky, and R.~M. Willett, ``Online markov decision processes
  with kullback--leibler control cost,'' \emph{IEEE Transactions on Automatic
  Control}, vol.~59, no.~6, pp. 1423--1438, 2014.

\bibitem{ito2022kullback}
K.~Ito and K.~Kashima, ``Kullback--leibler control for discrete-time nonlinear
  systems on continuous spaces,'' \emph{SICE Journal of Control, Measurement,
  and System Integration}, vol.~15, no.~2, pp. 119--129, 2022.

\bibitem{ok2018exploration}
J.~Ok, A.~Proutiere, and D.~Tranos, ``Exploration in structured reinforcement
  learning,'' \emph{Advances in Neural Information Processing Systems},
  vol.~31, 2018.

\bibitem{munos2023nash}
R.~Munos, M.~Valko, D.~Calandriello, M.~G. Azar, M.~Rowland, Z.~D. Guo,
  Y.~Tang, M.~Geist, T.~Mesnard, A.~Michi, \emph{et~al.}, ``Nash learning from
  human feedback,'' \emph{arXiv preprint arXiv:2312.00886}, 2023.

\bibitem{bernardini2011stabilizing}
D.~Bernardini and A.~Bemporad, ``Stabilizing model predictive control of
  stochastic constrained linear systems,'' \emph{IEEE Trans. Autom. Control},
  vol.~57, no.~6, pp. 1468--1480, 2011.

\bibitem{campi2018general}
M.~C. Campi, S.~Garatti, and F.~A. Ramponi, ``A general scenario theory for
  nonconvex optimization and decision making,'' \emph{IEEE Transactions on
  Automatic Control}, vol.~63, no.~12, pp. 4067--4078, 2018.

\bibitem{schildbach2015scenario}
G.~Schildbach and F.~Borrelli, ``Scenario model predictive control for lane
  change assistance on highways,'' in \emph{IEEE Intelligent Vehicles Symposium
  (IV)}, 2015, pp. 611--616.

\bibitem{hu2022sharp}
H.~Hu, K.~Nakamura, and J.~F. Fisac, ``{SHARP: Shielding-aware robust planning
  for safe and efficient human-robot interaction},'' \emph{IEEE Robotics and
  Automation Letters}, vol.~7, no.~2, pp. 5591--5598, 2022.

\bibitem{li2023scenario}
J.~Li, C.-Y. Chiu, L.~Peters, F.~Palafox, M.~Karabag, J.~Alonso-Mora,
  S.~Sojoudi, C.~Tomlin, and D.~Fridovich-Keil, ``Scenario-game admm: A
  parallelized scenario-based solver for stochastic noncooperative games,''
  \emph{arXiv preprint arXiv:2304.01945}, 2023.

\bibitem{eysenbach2021maximum}
B.~Eysenbach and S.~Levine, ``{Maximum entropy RL (provably) solves some robust
  RL problems},'' \emph{arXiv preprint arXiv:2103.06257}, 2021.

\bibitem{wu2020efficient}
Z.~Wu, L.~Sun, W.~Zhan, C.~Yang, and M.~Tomizuka, ``Efficient sampling-based
  maximum entropy inverse reinforcement learning with application to autonomous
  driving,'' \emph{IEEE Robotics and Automation Letters}, vol.~5, no.~4, pp.
  5355--5362, 2020.

\bibitem{pitis2020maximum}
S.~Pitis, H.~Chan, S.~Zhao, B.~Stadie, and J.~Ba, ``Maximum entropy gain
  exploration for long horizon multi-goal reinforcement learning,'' in
  \emph{International Conference on Machine Learning}.\hskip 1em plus 0.5em
  minus 0.4em\relax PMLR, 2020, pp. 7750--7761.

\bibitem{helbing1995social}
D.~Helbing and P.~Molnar, ``Social force model for pedestrian dynamics,''
  \emph{Physical review E}, vol.~51, no.~5, p. 4282, 1995.

\bibitem{shi2022motion}
S.~Shi, L.~Jiang, D.~Dai, and B.~Schiele, ``Motion transformer with global
  intention localization and local movement refinement,'' \emph{arXiv preprint
  arXiv:2209.13508}, 2022.

\bibitem{salzmann2020trajectron++}
T.~Salzmann, B.~Ivanovic, P.~Chakravarty, and M.~Pavone, ``Trajectron++:
  Dynamically-feasible trajectory forecasting with heterogeneous data,'' in
  \emph{ECCV}.\hskip 1em plus 0.5em minus 0.4em\relax Springer, 2020, pp.
  683--700.

\bibitem{zhou2022hivt}
Z.~Zhou, L.~Ye, J.~Wang, K.~Wu, and K.~Lu, ``Hivt: Hierarchical vector
  transformer for multi-agent motion prediction,'' in \emph{Proceedings of the
  IEEE/CVF Conference on Computer Vision and Pattern Recognition}, 2022, pp.
  8823--8833.

\bibitem{jia2023hdgt}
X.~Jia, P.~Wu, L.~Chen, Y.~Liu, H.~Li, and J.~Yan, ``Hdgt: Heterogeneous
  driving graph transformer for multi-agent trajectory prediction via scene
  encoding,'' \emph{IEEE transactions on pattern analysis and machine
  intelligence}, 2023.

\bibitem{seff2023motionlm}
A.~Seff, B.~Cera, D.~Chen, M.~Ng, A.~Zhou, N.~Nayakanti, K.~S. Refaat,
  R.~Al-Rfou, and B.~Sapp, ``Motionlm: Multi-agent motion forecasting as
  language modeling,'' in \emph{Proceedings of the IEEE/CVF International
  Conference on Computer Vision}, 2023, pp. 8579--8590.

\bibitem{jiang2023motiondiffuser}
C.~Jiang, A.~Cornman, C.~Park, B.~Sapp, Y.~Zhou, D.~Anguelov, \emph{et~al.},
  ``Motiondiffuser: Controllable multi-agent motion prediction using
  diffusion,'' in \emph{Proceedings of the IEEE/CVF Conference on Computer
  Vision and Pattern Recognition}, 2023, pp. 9644--9653.

\bibitem{Ngiam2021SceneTransformer}
J.~Ngiam, B.~Caine, V.~Vasudevan, Z.~Zhang, H.-T.~L. Chiang, J.~Ling,
  R.~Roelofs, A.~Bewley, C.~Liu, A.~Venugopal, D.~Weiss, B.~Sapp, Z.~Chen, and
  J.~Shlens, ``Scene transformer: A unified architecture for predicting
  multiple agent trajectories,'' in \emph{ICLR}, June 2021.

\bibitem{Varadarajan2021MultipathPlusPlus}
B.~Varadarajan, A.~Hefny, A.~Srivastava, K.~S. Refaat, N.~Nayakanti,
  A.~Cornman, K.~Chen, B.~Douillard, C.~P. Lam, D.~Anguelov, and B.~Sapp,
  ``{MultiPath++}: Efficient information fusion and trajectory aggregation for
  behavior prediction,'' Nov. 2021.

\bibitem{kumar2020interaction}
S.~Kumar, Y.~Gu, J.~Hoang, G.~C. Haynes, and M.~Marchetti-Bowick,
  ``Interaction-based trajectory prediction over a hybrid traffic graph,'' in
  \emph{IROS}.\hskip 1em plus 0.5em minus 0.4em\relax IEEE, 2021, pp.
  5530--5535.

\bibitem{sun2022m2i}
Q.~Sun, X.~Huang, J.~Gu, B.~C. Williams, and H.~Zhao, ``{M2I}: From factored
  marginal trajectory prediction to interactive prediction,'' in \emph{CVPR},
  2022, pp. 6543--6552.

\bibitem{ban2022deep}
Y.~Ban, X.~Li, G.~Rosman, I.~Gilitschenski, O.~Meireles, S.~Karaman, and
  D.~Rus, ``A deep concept graph network for interaction-aware trajectory
  prediction,'' in \emph{ICRA}.\hskip 1em plus 0.5em minus 0.4em\relax IEEE,
  2022, pp. 8992--8998.

\bibitem{lidard2023nashformer}
J.~Lidard, O.~So, Y.~Zhang, J.~DeCastro, X.~Cui, X.~Huang, Y.-L. Kuo,
  J.~Leonard, A.~Balachandran, N.~Leonard, and G.~Rosman, ``Nashformer:
  Leveraging local nash equilibria for semantically diverse trajectory
  prediction,'' 2023.

\bibitem{huang2023gameformer}
Z.~Huang, H.~Liu, and C.~Lv, ``Gameformer: Game-theoretic modeling and learning
  of transformer-based interactive prediction and planning for autonomous
  driving,'' \emph{arXiv preprint arXiv:2303.05760}, 2023.

\bibitem{huang2020diversitygan}
X.~Huang, S.~G. McGill, J.~A. DeCastro, L.~Fletcher, J.~J. Leonard, B.~C.
  Williams, and G.~Rosman, ``{DiversityGAN}: Diversity-aware vehicle motion
  prediction via latent semantic sampling,'' \emph{IEEE RA-L}, vol.~5, no.~4,
  pp. 5089--5096, 2020.

\bibitem{shiroshita2020behaviorally}
S.~Shiroshita, S.~Maruyama, D.~Nishiyama, M.~Y. Castro, K.~Hamzaoui, G.~Rosman,
  J.~DeCastro, K.-H. Lee, and A.~Gaidon, ``Behaviorally diverse traffic
  simulation via reinforcement learning,'' in \emph{IROS}.\hskip 1em plus 0.5em
  minus 0.4em\relax IEEE, 2020, pp. 2103--2110.

\bibitem{zhao2021tnt}
H.~Zhao, J.~Gao, T.~Lan, C.~Sun, B.~Sapp, B.~Varadarajan, Y.~Shen, Y.~Shen,
  Y.~Chai, C.~Schmid, \emph{et~al.}, ``{TNT}: Target-driven trajectory
  prediction,'' in \emph{Conference on Robot Learning}.\hskip 1em plus 0.5em
  minus 0.4em\relax PMLR, 2021, pp. 895--904.

\bibitem{huang2022hyper}
X.~Huang, G.~Rosman, I.~Gilitschenski, A.~Jasour, S.~G. McGill, J.~J. Leonard,
  and B.~C. Williams, ``{HYPER}: Learned hybrid trajectory prediction via
  factored inference and adaptive sampling,'' in \emph{ICRA}.\hskip 1em plus
  0.5em minus 0.4em\relax IEEE, 2022, pp. 2906--2912.

\bibitem{dixit2023adaptive}
A.~Dixit, L.~Lindemann, S.~X. Wei, M.~Cleaveland, G.~J. Pappas, and J.~W.
  Burdick, ``Adaptive conformal prediction for motion planning among dynamic
  agents,'' in \emph{Learning for Dynamics and Control Conference}.\hskip 1em
  plus 0.5em minus 0.4em\relax PMLR, 2023, pp. 300--314.

\bibitem{huang2023differentiable}
Z.~Huang, H.~Liu, J.~Wu, and C.~Lv, ``Differentiable integrated motion
  prediction and planning with learnable cost function for autonomous
  driving,'' \emph{IEEE transactions on neural networks and learning systems},
  2023.

\bibitem{zhong2023guided}
Z.~Zhong, D.~Rempe, D.~Xu, Y.~Chen, S.~Veer, T.~Che, B.~Ray, and M.~Pavone,
  ``Guided conditional diffusion for controllable traffic simulation,'' in
  \emph{2023 IEEE International Conference on Robotics and Automation
  (ICRA)}.\hskip 1em plus 0.5em minus 0.4em\relax IEEE, 2023, pp. 3560--3566.

\bibitem{zhong2023language}
Z.~Zhong, D.~Rempe, Y.~Chen, B.~Ivanovic, Y.~Cao, D.~Xu, M.~Pavone, and B.~Ray,
  ``Language-guided traffic simulation via scene-level diffusion,'' in
  \emph{Conference on Robot Learning}.\hskip 1em plus 0.5em minus 0.4em\relax
  PMLR, 2023, pp. 144--177.

\bibitem{huang2024versatile}
Z.~Huang, Z.~Zhang, A.~Vaidya, Y.~Chen, C.~Lv, and J.~F. Fisac, ``Versatile
  scene-consistent traffic scenario generation as optimization with
  diffusion,'' \emph{arXiv preprint arXiv:2404.02524}, 2024.

\bibitem{ccinlar2011probability}
E.~{\c{C}}inlar, \emph{Probability and stochastics}.\hskip 1em plus 0.5em minus
  0.4em\relax Springer, 2011.

\bibitem{donsker1983asymptotic}
M.~D. Donsker and S.~S. Varadhan, ``Asymptotic evaluation of certain markov
  process expectations for large time. iv,'' \emph{Communications on pure and
  applied mathematics}, vol.~36, no.~2, pp. 183--212, 1983.

\bibitem{dupuis2011weak}
P.~Dupuis and R.~S. Ellis, \emph{A weak convergence approach to the theory of
  large deviations}.\hskip 1em plus 0.5em minus 0.4em\relax John Wiley \& Sons,
  2011.

\bibitem{dupuis2019representations}
P.~Dupuis, ``Representations and weak convergence methods for the analysis and
  approximation of rare events,'' \emph{Padova notes}, 2019.

\bibitem{laine2023computation}
F.~Laine, D.~Fridovich-Keil, C.-Y. Chiu, and C.~Tomlin, ``The computation of
  approximate generalized feedback nash equilibria,'' \emph{SIAM Journal on
  Optimization}, vol.~33, no.~1, pp. 294--318, 2023.

\bibitem{bishop2006PRML}
C.~M. Bishop, \emph{Pattern Recognition and Machine Learning}.\hskip 1em plus
  0.5em minus 0.4em\relax Springer, 2006.

\bibitem{nocedal2006numerical}
J.~Nocedal and S.~Wright, \emph{Numerical optimization}.\hskip 1em plus 0.5em
  minus 0.4em\relax Springer Science \& Business Media, 2006.

\bibitem{mesbah2016stochastic}
A.~Mesbah, ``Stochastic model predictive control: An overview and perspectives
  for future research,'' \emph{IEEE Control Systems Magazine}, vol.~36, no.~6,
  pp. 30--44, 2016.

\bibitem{jax2018github}
\BIBentryALTinterwordspacing
J.~Bradbury, R.~Frostig, P.~Hawkins, M.~J. Johnson, C.~Leary, D.~Maclaurin,
  G.~Necula, A.~Paszke, J.~Vander{P}las, S.~Wanderman-{M}ilne, and Q.~Zhang,
  ``{JAX}: composable transformations of {P}ython+{N}um{P}y programs,'' 2018.
  [Online]. Available: \url{http://github.com/google/jax}
\BIBentrySTDinterwordspacing

\bibitem{zhang2020optimization}
X.~Zhang, A.~Liniger, and F.~Borrelli, ``Optimization-based collision
  avoidance,'' \emph{IEEE Transactions on Control Systems Technology}, vol.~29,
  no.~3, pp. 972--983, 2020.

\bibitem{dosovitskiy2017carla}
A.~Dosovitskiy, G.~Ros, F.~Codevilla, A.~Lopez, and V.~Koltun, ``Carla: An open
  urban driving simulator,'' in \emph{Conference on robot learning}.\hskip 1em
  plus 0.5em minus 0.4em\relax PMLR, 2017, pp. 1--16.

\bibitem{ross2011reduction}
S.~Ross, G.~Gordon, and D.~Bagnell, ``A reduction of imitation learning and
  structured prediction to no-regret online learning,'' in \emph{Proceedings of
  the fourteenth international conference on artificial intelligence and
  statistics}.\hskip 1em plus 0.5em minus 0.4em\relax JMLR Workshop and
  Conference Proceedings, 2011, pp. 627--635.

\bibitem{hsu2023sf}
K.-C. Hsu, H.~Hu, and J.~F. Fisac, ``The safety filter: A unified view of
  safety-critical control in autonomous systems,'' \emph{Annual Review of
  Control, Robotics, and Autonomous Systems}, 2023.

\bibitem{Nayakanti2022Wayformer}
N.~Nayakanti, R.~Al-Rfou, A.~Zhou, K.~Goel, K.~S. Refaat, and B.~Sapp,
  ``Wayformer: Motion forecasting via simple \& efficient attention networks,''
  July 2022.

\bibitem{chen2022scept}
Y.~Chen, B.~Ivanovic, and M.~Pavone, ``Scept: Scene-consistent, policy-based
  trajectory predictions for planning,'' in \emph{Proceedings of the IEEE/CVF
  Conference on Computer Vision and Pattern Recognition}, 2022, pp.
  17\,103--17\,112.

\bibitem{gulino2023waymax}
C.~Gulino, J.~Fu, W.~Luo, G.~Tucker, E.~Bronstein, Y.~Lu, J.~Harb, X.~Pan,
  Y.~Wang, X.~Chen, \emph{et~al.}, ``Waymax: An accelerated, data-driven
  simulator for large-scale autonomous driving research,'' \emph{arXiv preprint
  arXiv:2310.08710}, 2023.

\bibitem{treiber2000congested}
M.~Treiber, A.~Hennecke, and D.~Helbing, ``Congested traffic states in
  empirical observations and microscopic simulations,'' \emph{Physical review
  E}, vol.~62, no.~2, p. 1805, 2000.

\bibitem{ettinger2021large}
S.~Ettinger, S.~Cheng, B.~Caine, C.~Liu, H.~Zhao, S.~Pradhan, Y.~Chai, B.~Sapp,
  C.~R. Qi, Y.~Zhou, \emph{et~al.}, ``Large scale interactive motion
  forecasting for autonomous driving: The waymo open motion dataset,'' in
  \emph{Proceedings of the IEEE/CVF International Conference on Computer
  Vision}, 2021, pp. 9710--9719.

\bibitem{zhang2023cat}
L.~Zhang, Z.~Peng, Q.~Li, and B.~Zhou, ``Cat: Closed-loop adversarial training
  for safe end-to-end driving,'' in \emph{Conference on Robot Learning}.\hskip
  1em plus 0.5em minus 0.4em\relax PMLR, 2023, pp. 2357--2372.

\bibitem{sarter1997team}
N.~B. Sarter and D.~D. Woods, ``Team play with a powerful and independent
  agent: {{Operational}} experiences and automation surprises on the {{Airbus
  A-320}},'' vol.~39, no.~4, pp. 553--569.

\bibitem{jamieson2022b737}
G.~A. Jamieson, G.~Skraaning, and J.~Joe, ``The {{B737 MAX}} 8 accidents as
  operational experiences with automation transparency,'' vol.~52, no.~4.

\bibitem{fisacBHFWTD18}
J.~F. Fisac, A.~Bajcsy, S.~L. Herbert, D.~Fridovich{-}Keil, S.~Wang, C.~J.
  Tomlin, and A.~D. Dragan, ``Probabilistically safe robot planning with
  confidence-based human predictions,'' in \emph{Robotics: Science and
  Systems}, 2018.

\bibitem{tian2022safety}
R.~Tian, L.~Sun, A.~Bajcsy, M.~Tomizuka, and A.~D. Dragan, ``Safety assurances
  for human-robot interaction via confidence-aware game-theoretic human
  models,'' in \emph{ICRA}.\hskip 1em plus 0.5em minus 0.4em\relax IEEE, 2022,
  pp. 11\,229--11\,235.

\bibitem{aggarwal2024policy}
S.~Aggarwal, M.~Bastopcu, T.~Ba{\c{s}}ar, \emph{et~al.}, ``{Policy Optimization
  finds Nash Equilibrium in Regularized General-Sum LQ Games},'' \emph{arXiv
  preprint arXiv:2404.00045}, 2024.

\end{thebibliography}

\newpage  

\onecolumn
\appendix

\subsection{Proof of Lemma~\ref{lem:Dupuis}}
\label{apdx:proof_lem_Dupuis}
By a Lagrange multiplier argument, we guess the minimizer of \eqref{eqn: dupuis_lemma}  is given by
\begin{equation} \label{eqn: dupuis_ansatz}
    \frac{d\pi^*}{d\tilde \pi} = \frac{e^{-\qfunc(x,u)/\reg} }{\int_{\mathcal U} e^{-\qfunc(x, u)/\reg} d\tilde{\pi}}.
\end{equation}
Next, we rearrange the term inside the infimum, which we denote by $I^\pi$. 
\begin{equation}
    \begin{aligned}
        I^\pi &=\expectation^\policy \left[ \qfunc(x,u) \right] + \reg D_{KL} (\pi || \tilde{\pi}) \\
        &=\int_{\mathcal{U}} \qfunc(x,u) d\pi+ \reg D_{KL} (\pi || \tilde{\pi}) \\
        &= \int_{\mathcal{U}} \qfunc(x,u)d\pi + \reg \int_{\mathcal{U}} \log \left( \frac{d\pi}{d\tilde{\pi}} \right) d\pi\\
        &= \int_{\mathcal{U}} \qfunc(x,u) d\pi + \reg \int_{\mathcal{U}} \log \left( \frac{d\pi}{d\pi^*}\frac{d\pi^*}{d\tilde{\pi}} \right) d\pi\\
        &= \reg \int_{\mathcal{U}} \log ( e^{ (\qfunc(x,u)/\reg})d\pi + \reg \int_{\mathcal{U}} \log \left( \frac{d\pi^*}{d\tilde{\pi}} \right) d\pi + \reg D_{KL}(\pi || \pi^*)\\
        &= \reg \int_{\mathcal{U}} \log \left( e^{ (\qfunc(x,u)/\reg} \times \frac{e^{-\qfunc(x,u)/\reg} }{\int_{\mathcal U} e^{-Q(x, u)/\reg} d\tilde{\pi} }\right)d\pi + \reg D_{KL}(\pi || \pi^*)\\
        &= -\reg \int_{\mathcal{U}} \log \left( \int_{\mathcal{U}} e^{-\qfunc(x,u)/\reg} d\tilde{\pi} \right) d\pi + \reg D_{KL}(\pi || \pi^*)\\
        &= -\reg \log \left( \int_{\mathcal{U}} e^{-\qfunc(x,u)/\reg} d\tilde{\pi} \right) \underbrace{\int_{\mathcal{U}} d\pi}_{=1} + \reg D_{KL}(\pi || \pi^*)\\
        & \geq -\reg \log \left( \int_{\mathcal{U}} e^{-\qfunc(x,u)/\reg} d\tilde{\pi} \right),
    \end{aligned}
\end{equation}
%
%
where we substitute the definition \eqref{eqn: dupuis_ansatz} in the penultimate equality. Note that if $\pi << \tilde{\pi}$ and $\tilde{\pi} << \pi^*$, then $\pi << \pi^*$. Recalling that $\reg D_{KL}(\pi || \pi^*) \geq 0$, we attain the infimum $\inf_\pi I^\pi$  exactly when $\pi = \pi^*$.

\subsection{Proof of Theorem~\ref{thm:global_nash}}
\label{apdx:thm_global_nash}

We restate the theorem for convenience.

\noindent The $N$-player nonzero-sum KL-LQG dynamic game~\eqref{eqn: overall optimization} admits a unique global feedback Nash equilibrium solution if,
\begin{enumerate}
    \item The dynamics follow the functional form
    \begin{equation*}
        x_{t+1} = A_t\jxt + \sum_{i \in [N]} B_t^i u^i_t + \dstb_t,
    \end{equation*}
    where $x_{0} \sim \gaussian(\mean_{x_{0}}, \covar_{x_{0}})$, $\dstb_t \sim \gaussian(0, \covar_{\dstb})$,

    \item The costs have the functional form 
    \begin{equation*}
    \begin{aligned}
        J^i(\pi^i) =& \mathbb E^{ \pi } \Bigg[ \sum_{t=0}^T \frac{1}{2} \Big(\jxt^\top Q^i_t \jxt +
        \sum_{j \in [N]} {u_t^j}^\top R_t^{ij} u_t^j \Big) + \sum_{t=0}^T  \reg^i D_{KL}(\pi_t^{i} || \tilde \pi_t^i ) \Bigg ],
    \end{aligned}
    \end{equation*}
    where, for all $t \in [T]$, $Q^i_t \succeq 0$, $R_t^{ij} \succeq 0$, $\forall i,j \in [N], j \neq i$, $R_t^{ii} \succ 0$, $\forall i \in [N]$,

    \item The reference policies are Gaussian, i.e., $\policyref^i_t \sim \gaussian(\tilde{\mu}^i_t, \tilde{\Sigma}^i_t)$ for all $t \in [T]$ and $i \in [N]$. 
\end{enumerate}
Moreover, the global (mixed-strategy) Nash equilibrium is a set of time-varying policies $\pi_t^{i*} = \mathcal N(\mu_t^{i*}, \Sigma_t^{i*}),~\forall i \in [N]$ with mean and covariance given by
\begin{equation} 
\begin{aligned}
    \mu^{i*}_t &=  - K^i_t \state_t - \kappa^i_t,   \\
    \Sigma^{i*}_t &= \left[ \textcolor{red}{\frac{1}{\reg^i}} \left(R^{ii}_t + {B^i_t}^T Z^i_{t+1} B^i_t \right) \textcolor{red}{+ \left(\tilde{\covar}^i_t\right)^{-1}} \right]^{-1},
\end{aligned}
\end{equation}
where $(K^i_t, \kappa^i_t)$ is given by solving the coupled KL-regularized Riccati equation:
\begin{equation}
    \begin{aligned}
        \big[{\color{red} \reg^i (\Tilde{\Sigma}^i_t)^{-1}}+ R_t^{ii} + B_t^{i\top}\qterm &B_t^{i\top} \big] K^i_t +  B_t^{i\top}\qterm\sum_{j\neq i}B_t^j K^j_t = B_t^{i\top}\qterm A_t, \\
        \big[ {\color{red}\reg^i(\Tilde{\Sigma}^i_t)^{-1}} + R_t^{ii} + B_t^{i\top}\qterm &B_t^{i\top} \big]\kappa^i_t + B_t^{i\top}\qterm\sum_{j\neq i}B_t^j \kappa^j_t = B_t^{i\top}\lterm {\color{red} - \reg^i(\Tilde{\Sigma}^i_t)^{-1}\Tilde{\mu}^i_t}.
    \end{aligned}
    \end{equation}
Value function parameters $(Z^i_{t}, z^i_{t})$ are computed recursively backward in time as
\begin{equation}
\begin{aligned}
    Z_t^i &= Q_t^i+\sum_{j\in[N]} (K_t^{j})^T R_t^{i j} K_t^j+F_t^T Z_{t+1}^i F_t {\color{red}+ \reg^i K^{i\top}_t(\Tilde{\Sigma}^i_t)^{-1}K^i_t},\\
    z_t^i &= \sum_{j\in[N]} (K_t^{j})^T R_t^{i j} \kappa^j_t+F_t^T\left(z_{t+1}^i+Z_{t+1}^i \beta_t\right) {\color{red}+ \reg^i K^{i\top}_t(\Tilde{\Sigma}^i_t)^{-1}(\kappa^i_t - \Tilde{\mu}^i_t)},
\end{aligned}
\end{equation}
with terminal conditions $Z^i_{T+1}=0$ and $z^i_{T+1}=0$.
Here, $F_t = A_t-\sum_{j\in[N]} B_t^j K_t^j$ and $\beta_t=-\sum_{j\in [N]} B_t^j \kappa_t^j$ for all $t \in [T]$.

\begin{proof}
    We proceed via induction. The base case $(t=T)$ follows directly from the hypothesis of the cost functional form in the theorem statement. For the induction step, we assume that the state-value function is a quadratic form at time $t+1$,
    \begin{equation}
        V_{t+1}^{i*} = \frac{1}{2}x_{t+1}^\top Z_{t+1}^i x_{t+1} + z_{t+1}^{i\top}x_t + \eta^i_{t+1},
    \end{equation}
    and we show this holds at time $t$. Denote $\pi_t = (\pi^i_t, \pi^{\neg i}_t)$. Therefore, we may write 
    \begin{equation}
        \begin{aligned}
            V_t^{i*}(x_t) &= \min_{\pi_t^i} \mathop{\mathbb{E}^{\pi_t}} \left[ \frac{1}{2} \left(x_t^\top Q^i_tx_t + \sum_{j\in[N]} u_t^{j\top} R^{ij}_tu_t^{j} \right) + \reg^i D_{KL}\left( \pi_t^i || \tilde{\pi}^i_t \right) +  \mathop{\mathbb{E}^{x_{t+1}}}\left[ V_{t+1}^{i*}(x_{t+1}) \right] \right]\\
             &= \min_{\pi_t^i} \mathop{\mathbb{E}^{\pi_t}} \left[ \underbrace{\frac{1}{2} \left(x_t^\top Q^i_tx_t + \sum_{j\in[N]} u_t^{j\top} R^{ij}_tu_t^{j} \right)}_{\text{(1)}} + \underbrace{\reg^i D_{KL}\left( \pi_t^i || \tilde{\pi}^i_t \right)}_{\text{(2)}} + \underbrace{ \mathop{\mathbb{E}^{x_{t+1}}}\left[ V_{t+1}^{i*}\left(A_tx_t + \sum_{j\in [N]} B^j_tu^j_t + d_t\right) \right]}_{\text{(3)}} \right].
             \label{eq:valuestar}
        \end{aligned}
    \end{equation}
    Let us now consider each term in the brackets:
    \begin{equation}
            \text{(1)} = \mathop{\mathbb{E}^{\pi_t}} \left[ \frac{1}{2} \left(x_t^\top Q^i_tx_t + \sum_{j\in[N]} u_t^{j\top} R^{ij}_tu_t^{j} \right) \right] = \frac{1}{2} x_t^\top Q^i_tx_t + \frac{1}{2}\left( \sum_{j\in[N]} \mu_t^{j\top}R_t^{ij}\mu_t^j + \trace \left(R_t^{ij}\Sigma_t^j \right) \right),
            \label{eq1}
    \end{equation}
    \begin{equation}
        \text{(2)} = \mathop{\mathbb{E}^{\pi_t}} \left[ \reg^i D_{KL}\left( \pi_t^i || \tilde{\pi}^i_t \right) \right] = \frac{\reg^i}{2} \left(n_{u^i} - \log\det (\Sigma_t^i) + \log\det(\tilde{\Sigma}_t^i) + \trace \left( \left( \tilde{\Sigma}_t^i\right)^{-1} \Sigma_t^i \right) + \left( \mu_t^i - \tilde{\mu}_t^i  \right)^\top \left( \tilde{\Sigma}_t^i\right)^{-1} \left( \mu_t^i - \tilde{\mu}_t^i  \right)\right),
        \label{eq2}
    \end{equation}

    \begin{equation}
        \text{(3)} 
        = \mathop{\mathbb{E}^{\pi_t}} \left[ \mathop{\mathbb{E}^{\pi_t,d}}\left[ \frac{1}{2} \left(A_tx_t + \sum_{j\in [N]} B^j_tu^j_t + d_t \right)^\top Z_{t+1}^i \left(A_tx_t + \sum_{j\in [N]} B^j_tu^j_t + d_t \right) + z_{t+1}^{i\top} \left(A_tx_t + \sum_{j\in [N]} B^j_tu^j_t + d_t \right) + \eta_{t+1}^i \right] \right].
        \label{eq3}
    \end{equation}
Expanding the terms inside the brackets of Equation \eqref{eq3}, taking expectations, and dropping terms not affecting the minimum, \eqref{eq3} yields
%
%
\begin{equation}
    \frac{1}{2} \left(\sum_{j\in [N]}B_t^j \mu_t^j \right)^\top Z_{t+1}^i \left(\sum_{j\in [N]}B_t^j \mu_t^j \right) + \sum_{j\in [N]} \left(A_tx_t \right)^\top Z_{t+1}^i B_t^j \mu_t^j + \frac{1}{2} \trace \left( Z_{t+1}^i \left( \Sigma_d + \sum_{j\in[N]}B_t^{j\top}\Sigma_t^j B_t^j \right)\right) + \sum_{j \in[N]} z_{t+1}^{i\top} B_t^j\mu_t^j.
    \label{eq3_short}
\end{equation}
Similarly, combining Equations \eqref{eq1}, \eqref{eq2}, \eqref{eq3_short} and dropping terms not affecting the minimum, we have 
\begin{equation}
    \begin{aligned}
        V_t^{i*}(x_t) &= \frac{1}{2}\left( \sum_{j\in[N]} \mu_t^{j\top}R_t^{ij}\mu_t^j + \trace \left(R_t^{ij}\Sigma_t^j \right) \right) \\
        &+\frac{\reg^i}{2} \left(- \log\det (\Sigma_t^i) + \log\det(\tilde{\Sigma}_t^i) + \trace \left( \left( \tilde{\Sigma}_t^i\right)^{-1} \Sigma_t^i \right) + \left( \mu_t^i - \tilde{\mu}_t^i  \right)^\top \left( \tilde{\Sigma}_t^i\right)^{-1} \left( \mu_t^i - \tilde{\mu}_t^i  \right)\right)\\
        &+ \frac{1}{2} \left(\sum_{j\in [N]}B_t^j \mu_t^j \right)^\top Z_{t+1}^i \left(\sum_{j\in [N]}B_t^j \mu_t^j \right) + \sum_{j\in [N]} \left(A_tx_t \right)^\top Z_{t+1}^i B_t^j \mu_t^j\\
        &+ \frac{1}{2} \trace \left( Z_{t+1}^i \left( \Sigma_d + \sum_{j\in[N]}B_t^{j\top}\Sigma_t^j B_t^j \right)\right) + \sum_{j \in[N]} z_{t+1}^{i\top} B_t^j\mu_t^j,
    \end{aligned}
    \label{eq:value_function_expected}
\end{equation}
which is convex quadratic in $\mu_t^i$ and convex in $\Sigma_t^i$. Therefore, by the first-order optimality conditions, we may take the gradient of Equation \eqref{eq:value_function_expected} and equate it to zero to find the optimal mean controls and the corresponding covariances. The optimal control law is linear state-feedback in the form of
\begin{equation}
    \mu_t^{i*} = -K_t^ix_t - \kappa_t^i.
    \label{eq:optimal-controls}
\end{equation}
Plugging \eqref{eq:optimal-controls} into the gradient, with respect to mean controls, of \eqref{eq:value_function_expected} and equating to zero yields coupled Riccati equations
\begin{equation}
    \begin{aligned}
        &\left(R_t^{ii} + \reg^i \left( \tilde{\Sigma}_t^i \right)^{-1} + B_t^{i\top}Z_{t+1}B_t^i \right)K_t^i + B_t^{i\top} Z_{t+1}^i \sum_{i\neq j \in [N]}B_t^j K_t^j = B_t^{i\top}Z_{t+1}^iA_t,\\
        &\left(R_t^{ii} + \reg^i \left( \tilde{\Sigma}_t^i \right)^{-1} + B_t^{i\top}Z_{t+1}B_t^i \right)\kappa_t^i + B_t^{i\top} Z_{t+1}^i \sum_{i\neq j \in [N]}B_t^j \kappa_t^j = z_{t+1}^{i\top}B_t^i - \reg^i \left( \tilde{\Sigma}_t^i \right)^{-1} \tilde{\mu}_t^i.
    \end{aligned}
\end{equation}
Repeating the same steps above, but this time imposing first-order optimality conditions with respect to the covariance $\Sigma_t^{i}$ yields

\begin{equation}
    R_t^{ii} + \reg^i (\tilde{\Sigma}_t^i)^{-1} - \reg^i(\Sigma_t^i)^{-1} + B_t^{i\top}Z_{t+1}^i B_t^{i} = 0,
\end{equation}
thereby giving the optimal covariance as 
\begin{equation}
    \Sigma_t^{i*} = \left[ \frac{1}{\reg^i} \left( R_t^{ii}+B_t^{i\top}Z_{t+1}^iB_t^i + \reg^i \left(\tilde{\Sigma}^i_t \right)^{-1} \right) \right]^{-1}.
\end{equation}
To solve the Riccati equations, it is necessary to derive expressions for $Z_{t}^i$ and $z_{t}^i$, which can be accomplished by taking expectations in \eqref{eq:valuestar} and rearranging with the optimal control law of \eqref{eq:optimal-controls}. Note that solving for $\eta_t^i$ is not necessary, as it will not affect the optimal policy. Following the lines of \cite{fridovich2020efficient}, we remove the $\min_{\pi_t^i}$ since we already imposed first-order optimality via our control scheme:
\begin{equation}
    \begin{aligned}
        V_t^{i*}(x_t)
             = & \mathop{\mathbb{E}^{\pi^*_t}} \frac{1}{2} \left(x_t^\top Q^i_tx_t + \sum_{j\in[N]} \left( -K_t^jx_t - \kappa_t^j \right)^{\top} R^{ij}_t\left( -K_t^jx_t - \kappa_t^j \right) \right)\\
             &+ \reg^i D_{KL}\left( \pi_t^{i*} || \tilde{\pi}^{i*}_t \right)\\
             &+  \mathop{\mathbb{E}^{\pi^*_t}} \mathop{\mathbb{E}^{x_{t+1}}}\frac{1}{2} \left(A_tx_t + \sum_{j\in [N]} B^j_t\left( -K_t^jx_t - \kappa_t^j \right) + d_t \right)^\top Z_{t+1}^i \left(A_tx_t + \sum_{j\in [N]} B^j_t\left( -K_t^jx_t - \kappa_t^j \right) + d_t \right) \\
             &+ \mathop{\mathbb{E}^{\pi^*_t}} \mathop{\mathbb{E}^{x_{t+1}}} z_{t+1}^{i\top} \left(A_tx_t + \sum_{j\in [N]} B^j_t\left( -K_t^jx_t - \kappa_t^j \right) + d_t \right) + \eta_{t+1}^i .
    \end{aligned}
\end{equation}
Taking expectations and rearranging yields the recursion parameters for the quadratic value function, which we now show. Firstly, let 
\begin{equation}
    \begin{aligned}
        F_t &= A_t - \sum_{j\in[N]}B_t^jK_t^j,\\
        \beta_t &= - \sum_{j\in[N]}B_t^j\kappa_t^j.
    \end{aligned}
\end{equation}
Then we have 
\begin{equation}
    \begin{aligned}
    V_t^{i*}(x_t)
             =&
        \frac{1}{2}x_t^\top \left[Q_t^i + \sum_{j\in[N]}K_t^{j\top}R_t^{ij}K_t^j + F_t^\top Z_{t+1}^iF_t + \reg^i K_t^{i\top}\left(\tilde{\Sigma}_t^i \right)^{-1}K_t^i \right]x_t\\
        &+ \left[\sum_{j\in [N]} K_t^{j\top}R_t^{ij}\kappa_t^j + F_t^\top \left(z_{t+1}^i + Z_{t+1}^i\beta_t \right)^\top + \reg^i K_t^{i\top}\left(\tilde{\Sigma}_t^i \right)^{-1} \left(\kappa_t^i - \tilde{\mu}_t^i\right) \right]x_t\\
        &+ \cdots , 
    \end{aligned}
\end{equation}
where $\cdots$ denotes the constant values pertaining relevant for calculating the $\eta_{t}^i$ terms, but are unnecessary for deriving the optimal policy. The terms in the brackets are precisely the recursions needed:
\begin{equation}
        \begin{aligned}
            Z_t^i &= Q_t^i + \sum_{j\in[N]}K_t^{j\top}R_t^{ij}K_t^j + F_t^\top Z_{t+1}^iF_t + \reg^i K_t^{i\top}\left(\tilde{\Sigma}_t^i \right)^{-1}K_t^i, ~&&Z^i_{T+1}=0, \\
    z_t^i &= \sum_{j\in [N]} K_t^{j\top}R_t^{ij}\kappa_t^j + F_t^\top \left(z_{t+1}^i + Z_{t+1}^i\beta_t \right)^\top + \reg^i K_t^{i\top}\left(\tilde{\Sigma}_t^i \right)^{-1} \left(\kappa_t^i - \tilde{\mu}_t^i\right),~&&z^i_{T+1}=0,\\
        \end{aligned}
    \end{equation}
thereby completing our proof by construction.
Finally, we note that, in the above proof, we have assumed that the coupled Riccati equations are uniquely solvable to ensure a unique FBNE. A more detailed discussion on the solution uniqueness of coupled Riccati equations can be found in concurrent work~\cite{aggarwal2024policy}.
\end{proof}

\subsection{Proof of Proposition~\ref{prop:cl_guide}}
\label{apdx:prop_cl_guide}

The proof follows the same structure as in  \autoref{thm:global_nash}. We develop the expressions of \eqref{eq:valuestar} with the expressions for \eqref{eq1} and \eqref{eq3} remaining the same. The expression in \eqref{eq2} needs to be modified to account for the time dependence in the reference policy. We have:

    \begin{equation}
        \begin{aligned}
            \text{(2)}~\;\;  \reg^i D_{KL}&(\pi_t^{i} || \tilde \pi_t^i ) = \frac{\reg^i}{2}\Big[
       n_{u^i} - \logdet \Sigma^i_t + \logdet \Tilde{\Sigma}^i_t +\trace((\Tilde{\Sigma}^{i}_t)^{-1}\Sigma^i_t) +  (\mu^{i*} - \Tilde{\mu}^{i*})^\top (\Tilde{\Sigma}^i_t)^{-1} (\mu^{i*} - \Tilde{\mu}^{i*})\Big] \\
       =&\frac{\reg^i}{2}\Big[
       n_{u^i} - \logdet \Sigma^i_t + \logdet \Tilde{\Sigma}^i_t +\trace((\Tilde{\Sigma}^{i}_t)^{-1}\Sigma^i_t) +  (\mu^{i*} + \Tilde{K}^i_t\state + \Tilde{\kappa}^i_t)^\top (\Tilde{\Sigma}^i_t)^{-1} (\mu^{i*} + \Tilde{K}^i_t\state + \Tilde{\kappa}^i_t)\Big]  \\
       =&\frac{\reg^i}{2}\Big[
       n_{u^i} - \logdet \Sigma^i_t + \logdet \Tilde{\Sigma}^i_t +\trace((\Tilde{\Sigma}^{i}_t)^{-1}\Sigma^i_t) + \mu^{i*\top} (\Tilde{\Sigma}^i_t)^{-1} \mu^{i*} +  (\Tilde{K}^i_tx_t+\Tilde{\kappa}^i_t)^\top(\Tilde{\Sigma}^i_t)^{-1}(\Tilde{K}^i_tx_t+\Tilde{\kappa}^i_t)\\
       & + \mu^{i*\top} (\Tilde{\Sigma}^i_t)^{-1} (\Tilde{K}^i_tx_t + \Tilde{\kappa}^i_t)+ (\Tilde{K}^i_t x_t+\Tilde{\kappa}^i_t)^\top(\Tilde{\Sigma}^i_t)^{-1} \mu^{i*}\Big].
        \end{aligned}
    \end{equation}
    Following the same arguments as in Appendix~\ref{apdx:thm_global_nash}, we take derivatives and enforce them to be $0$ to obtain the following set of coupled Riccati equations:
    \begin{equation}
    \begin{aligned}
    \label{eqn: ricatti}
        \big[ \reg^i (\Tilde{\Sigma}^i)^{-1} + R_t^{ii} + B^{i\top}\qterm &B^{i\top} \big] K^i_t +  B^{i\top}\qterm\sum_{j\neq i}B^j K^j_t = B^{i\top}\qterm A  + \reg^i(\Tilde{\Sigma}^i_t)^{-1}\Tilde{K}^i_t,\\
        \big[ \reg^i(\Tilde{\Sigma}^i)^{-1} + R_t^{ii} + B^{i\top}\qterm &B^{i\top} \big]\kappa^i_t + B^{i\top}\qterm\sum_{j\neq i}B^j \kappa^j_t = B^{i\top}\lterm + \reg^i(\Tilde{\Sigma}^i_t)^{-1}\Tilde{\kappa}^i_t.
    \end{aligned}
    \end{equation}
    Substituting the formulas for $\mu^i_t$ back into the expression for $V_t^i(\state)$ we obtain the formulas in \eqref{eq:cl_update}.

\end{document}